\documentclass[final]{cvpr}
\pdfoutput=1
\usepackage{times}
\usepackage{xcolor}
\usepackage{epsfig}
\usepackage{graphicx}
\usepackage{amsmath}
\usepackage{amsthm}
\usepackage{amssymb}
\usepackage{multirow}
\usepackage[linesnumbered,titlenumbered,ruled,vlined,commentsnumbered,noend]{algorithm2e}
\usepackage{booktabs}
\usepackage{makecell}
\usepackage{caption}
\usepackage{subcaption}

\makeatletter
\@namedef{ver@everyshi.sty}{}
\makeatother
\usepackage{tikz}
\newenvironment{myalign}{\par\nobreak\small\noindent\align}{\endalign}
\newenvironment{myequation}{\par\nobreak\small\noindent\equation}{\endequation}

\newcommand{\bheading}[1]{{\vspace{2pt}\noindent{\textbf{#1}}\hspace{2pt}}} 

\newcommand{\tianwei}[1]{\textbf{\textcolor{red}{[Tianwei: #1]}}}

\newtheorem{definition}{Definition}
\newtheorem{theorem}{Theorem}
\newenvironment{packeditemize}{
\begin{list}{$\bullet$}{
\setlength{\labelwidth}{8pt}
\setlength{\itemsep}{0pt}
\setlength{\leftmargin}{\labelwidth}
\addtolength{\leftmargin}{\labelsep}
\setlength{\parindent}{0pt}
\setlength{\listparindent}{\parindent}
\setlength{\parsep}{0pt}
\setlength{\topsep}{3pt}}}{\end{list}}

\usepackage[pagebackref=true,breaklinks=true,colorlinks,bookmarks=false]{hyperref}

\def\cvprPaperID{****} 

\def\cvprPaperID{3665} 
\def\confYear{CVPR 2021}

\begin{document}
\title{Privacy-preserving Collaborative Learning with \\Automatic Transformation Search}  


\author {
Wei Gao$^{1,2}$, ~~ Shangwei Guo$^{3}$, ~~ Tianwei Zhang\thanks{Corresponding author.} $^{1}$, ~~ Han Qiu$^{4}$, ~~ Yonggang Wen$^{1}$, ~~ Yang Liu$^{1}$ \\
$^1${Nanyang Technological University,}  ~~ $^2${S-Lab, Nanyang Technological University,} \\ $^3${Chongqing University,} ~~ $^4${Tsinghua University} \\
  \tt\small{\{gaow0007,tianwei.zhang,ygwen,yangliu\}@ntu.edu.sg,swguo@cqu.edu.cn,qiuhan@tsinghua.edu.cn}
}

\maketitle

\begin{abstract}
Collaborative learning has gained great popularity due to its benefit of data privacy protection: participants can jointly train a Deep Learning model without sharing their training sets. However, recent works discovered that an adversary can fully recover the sensitive training samples from the shared gradients. Such reconstruction attacks pose severe threats to collaborative learning. Hence, effective mitigation solutions are urgently desired. 

In this paper, we propose to leverage data augmentation to defeat reconstruction attacks: by preprocessing sensitive images with carefully-selected transformation policies, it becomes infeasible for the adversary to extract any useful information from the corresponding gradients. We design a novel search method to automatically discover qualified policies. We adopt two new metrics to quantify the impacts of transformations on data privacy and model usability, which can significantly accelerate the search speed. Comprehensive evaluations demonstrate that the policies discovered by our method can defeat existing reconstruction attacks in collaborative learning, with high efficiency and negligible impact on the model performance. The code is available at \href{https://tinyurl.com/23kj554v}{https://tinyurl.com/23kj554v}.

\end{abstract}

\section{Introduction}
A collaborative learning system enables multiple participants to jointly train a shared Deep Learning (DL) model for a common artificial intelligence task \cite{yang2019federated,melis2019exploiting,guo2020towards}. Typical collaborative systems are distributed systems such as federated learning systems, where each participant iteratively calculates the local gradients based on his own training dataset and shares them with other participants to approach the ideal model. This collaborative mode can significantly improve the training speed, model performance and generalization. Besides, it can also protect the training data privacy, as participants do not need to release their sensitive data during the training phase. Due to these advantages, collaborative learning has become promising in many scenarios, e.g., smart manufacturing \cite{hao2019efficient}, autonomous driving \cite{niknam2020federated}, digital health \cite{brisimi2018federated}, etc.


Although each participant does not disclose the training dataset, he has to share with others the gradients, which can leak information of the sensitive data indirectly. 
Past works \cite{hitaj2017deep,melis2019exploiting,nasr2019comprehensive} demonstrated the possibility of membership inference and property inference attacks in collaborative learning. A more serious threat is the \emph{reconstruction attack} \cite{zhu2019deep,zhao2020idlg,geiping2020inverting}, where an adversary can recover the exact values of samples from the shared gradients with high fidelity. This attack is very practical under realistic and complex circumstances (e.g., large-size images, batch training). 


Due to the severity of this threat, an effective and practical defense solution is desirable to protect the privacy of collaborative learning. Common privacy-aware solutions \cite{zhu2019deep,wei2020framework} attempt to increase the difficulty of input reconstruction by obfuscating the gradients. However, the obfuscation magnitude is bounded by the performance requirement of the DL task: a large-scale obfuscation can hide the input information, but also impair the model accuracy. The effectiveness of various techniques (e.g., noise injection, model pruning) against reconstruction attacks have been empirically evaluated \cite{zhu2019deep}. Unfortunately, they cannot achieve a satisfactory trade-off between data privacy and model usability, and hence become less practical.



Motivated by the limitations of existing solutions, this paper aims to solve the privacy issue from a different perspective: \emph{obfuscating the training data to make the reconstruction difficult or infeasible.} The key insight of our strategy is to \emph{repurpose data augmentation techniques for privacy enhancement}. A variety of transformation approaches have been designed to improve the model performance and generalization. We aim to leverage certain transformation functions to preprocess the training sets and then train the gradients, which can prevent malicious participants from reconstructing the transformed or original samples.



Mitigating reconstruction attacks via data augmentation is challenging. First, existing image transformation functions are mainly used for performance and generalization improvement. It is unknown which ones are effective in reducing information leakage. Second, conventional approaches apply these transformations to augment the training sets, where original data samples are still kept for model training. This is relatively easier to maintain the model performance. In contrast, to achieve our goal, we have to abandon the original samples, and only use the transformed ones for training, which can impair the model accuracy.


We introduce a systematic approach to overcome these challenges. 
Our goal is to automatically discover an ensemble of effective transformations from a large collection of commonly-used data augmentation functions. This ensemble is then formed as a transformation policy, to preserve the privacy of collaborative learning. Due to the large search space and training overhead, it is computationally infeasible to evaluate the privacy and performance impacts of each possible policy. Instead, we design two novel metrics to quantify the policies without training a complete model. These metrics with our new search algorithm can identify the optimal policies within 2.5 GPU hours. 


The identified transformation policies exhibit great capability of preserving the privacy while maintaining the model performance. They also enjoy the following properties: (1) the policies are general and able to defeat different variants of reconstruction attacks. (2) The input transformations are performed locally without modifying the training pipeline. They are applicable to any collaborative learning systems and algorithms. (3) The transformations are lightweight with negligible impact on the training efficiency. (4) The policies have high transferability: the optimal policy searched from one dataset can be directly applied to other datasets as well. 










\section{Related Work}
\subsection{Reconstruction Attacks}
In collaborative learning, reconstruction attacks aim to recover the training samples from the shared gradients. Zhu et al. \cite{zhu2019deep} first formulated this attack as an optimization process: the adversarial participant searches for the optimal samples in the input space that can best match the gradients. L-BFGS~\cite{liu1989limited} was adopted to realize this attack. Following this work, several improved attacks were further proposed to enhance the attack effects and reduce the cost. For instance, Zhao et al. \cite{zhao2020idlg} first extracted the training labels from the gradients, and then recovered the training samples with higher convergence speed. Geiping et al. \cite{geiping2020inverting} adopted the cosine similarity as the distance function and Adam as the optimizer to solve the optimization problem, which can yield more precise reconstruction results. He et al. \cite{he2019model,he2020attacking} proposed reconstruction attacks against collaborative inference systems, which are not considered in this paper.

\subsection{Existing Defenses and Limitations}
One straightforward defense strategy is to obfuscate the gradients before releasing them, in order to make the reconstruction difficult or infeasible. Differential privacy is a theoretical framework to guide the randomization of exchange information~\cite{abadi2016deep, lecuyer2019certified, phan2020scalable,guo2020differentially}.  For instance, Zhu et al. \cite{zhu2019deep} tried to add Gaussian/Laplacian noises guided by differential privacy to the gradients, and compress the model with gradient pruning. Unfortunately, there exists an unsolvable conflict between privacy and usability in these solutions: a large-scale obfuscation can compromise the model performance while a small-scale obfuscation still leaks certain amount of information. Such ineffectiveness of these methods was validated in \cite{zhu2019deep}, and will be further confirmed in this paper  (Table \ref{tab:defense}). Wei et al. \cite{wei2020framework} proposed to adjust the hyper-parameters (e.g. bath size, loss or distance function), which also has limited impact on the attack results. 


An alternative direction is to design new collaborative learning systems to thwart the reconstruction attacks. Zhao et al. \cite{zhao2020privatedl} proposed a framework that transfers sensitive samples to public ones with privacy protection, based on which the participants can collaboratively update their local models with noise‐preserving labels. Fan et al. \cite{fan2020rethinking} designed a secret polarization network for each participant to produce secret losses and calculate the gradients. 
These approaches require all participants to follow the new training pipelines or optimization methods. They cannot be directly applied to existing collaborative implementations. This significantly restricts their practicality.



\section{Problem Statement}
\subsection{System Model}
We consider a standard collaborative learning system where all participants jointly train a global model $M$. Each participant owns a private dataset $D$. Let $\mathcal{L}, W$ be the loss function and the parameters of $M$, respectively. At each iteration, every participant randomly selects a training sample $(x, y)$, 
calculates the loss $\mathcal{L}(x, y)$ by forward propagation and then the gradient $\bigtriangledown W(x, y) = \frac{\partial \mathcal{L}(x, y)}{\partial W}$ using backward propagation. The participants can also use the mini-batch SGD, where a mini-batch of samples are randomly selected to train the gradient at each iteration. 

Gradients need to be consolidated at each iteration. In a centralized system, a parameter server aggregates all the gradients, and sends the updated one to each participant. In a decentralized system, each participant aggregates the gradients from his neighbors, and then broadcasts the results. 



\subsection{Attack Model}

We consider a honest-but-curious adversarial entity in the collaborative learning system, who receives other participants' gradients in each iteration, and tries to reconstruct the private training samples from them. In the centralized mode, this adversary is the parameter server, while in the decentralized mode, the adversary can be an arbitrary participant. 

Common reconstruction techniques \cite{zhu2019deep,zhao2020idlg,geiping2020inverting} adopt different optimization algorithms to extract training samples from the gradients. Specifically, given a gradient $\bigtriangledown W(x, y)$, the attack goal is to discover a pair of sample and label $(x', y')$, such that the corresponding gradient $\bigtriangledown W(x', y')$ is very close to $\bigtriangledown W$. This can be formulated as an optimization problem of minimizing the objective:

\newcommand{\argmin}{\operatornamewithlimits{argmin}} 
\begin{equation}
\label{eq:attack}
	x^*, y^* = \argmin_{x', y'} \quad  ||\bigtriangledown W(x, y) - \bigtriangledown W(x', y')||,
\end{equation}
where $||\cdot||$ is a norm for measuring the distance between the two gradients. A reconstruction attack succeeds if the identified $x^*$ is visually similar to $x$. This can be quantified by the metric of Peak Signal-to-Noise Ratio (PSNR)~\cite{hore2010image}. Formally, a reconstruction attack is defined as below:

\begin{definition}(($\epsilon, \delta$)-Reconstruction Attack)
    Let ($x^*$, $y^*$) be the solution to Equation \ref{eq:attack}, and ($x$, $y$) be the target training sample that produces $\bigtriangledown W(x, y)$. This process is called a ($\epsilon, \delta$)-reconstruction attack if the following property is held:
    \begin{equation}
        \emph{Pr}[\emph{\texttt{PSNR}}(x^*, x)\geq \epsilon] \geq 1-\delta.
    \end{equation}
\end{definition}

\section{Methodology}
\subsection{Overview}

Driven by the severity of reconstruction attacks, and limitations of existing defenses, we focus on a new mitigation opportunity in this paper: \emph{transforming the sensitive training samples to make the reconstruction difficult or even infeasible}. Image transformation has been widely adopted to mitigate adversarial examples \cite{qiu2020fencebox,qiu2020mitigating,qiu2020towards}, backdoor attacks \cite{zeng2020deepsweep}, and attack watermarking schemes \cite{guo2020hidden}. We repurpose it for defeating reconstruction attacks. Specifically, given a private dataset $D$, we aim to find a policy composed of a set of transformation functions $T = t_1 \circ t_2 \circ ... \circ t_n$, to convert each sample $x\in D$ to $\widehat{x} = T(x)$ and establish a new dataset $\widehat{D}$. The data owner can use $\widehat{D}$ to calculate the gradients and safely share them with untrusted collaborators in collaborative learning. Such a transformation policy must satisfy two requirements: (1) the adversarial participant is not able to infer $\widehat{x}$ (and $x$) from $\bigtriangledown W(\widehat{x}, y)$. (2) The final model should maintain similar performance as the one trained from $D$. 
We formally define our strategy as below:

\begin{definition}(($\epsilon, \delta, \gamma$)-Privacy-aware Transformation Policy)
    Given a dataset $D$, and an ensemble of transformations $T$, let $\widehat{D}$ be another dataset transformed from $D$ with $T$. Let $M$ and $\widehat{M}$ be the models trained over $D$ and $\widehat{D}$, respectively. $T$ is defined to be ($\epsilon, \delta, \gamma$)-privacy-aware, if the following two requirements are met:
    \begin{myalign}
        & \emph{Pr}[\emph{\texttt{PSNR}}(x^*, \widehat{x})< \epsilon] \geq 1-\delta, \forall x \in D, \\
        & \emph{\texttt{ACC}}(M) - \emph{\texttt{ACC}}(\widehat{M}) < \gamma,
    \end{myalign}
    where $\widehat{x}=T(x)$, $x^*$ is the reconstructed input from $\bigtriangledown W(\widehat{x}, y)$, and \emph{\texttt{ACC}} is the prediction accuracy function.
\end{definition}

It is critical to identify the transformation functions that can satisfy the above two requirements. With the advance of computer vision, different image transformations have been designed for better data augmentation. We aim to repurpose some of these data augmentation approaches to enhance the privacy of collaborative learning. 

Due to the large quantity and variety of augmentation functions, we introduce a systematic and automatic method to search for the most privacy-preserving and efficient policy. Our idea is inspired by AutoAugment \cite{iccv2019autoaugment}, which exploited AutoML techniques~\cite{wistuba2019survey} to automatically search for optimal augmentation policies to improve the model accuracy and generalization. However, it is difficult to apply this solution directly to our privacy problem. We need to address two new challenges: (1) how to efficiently evaluate the satisfaction of the two requirements for each policy (Sections \ref{sec:performance-metric} and \ref{sec:search-train}); and (2) how to select the appropriate search space and sampling method (Section \ref{sec:search-train}).

\subsection{Privacy Score}
\label{sec:privacy-metric}

During the search process, we need to quantify the privacy effect of the candidate policies. The PSNR metric is not efficient here, as it requires to perform an end-to-end reconstruction attack over a well-trained model. Instead, we design a new privacy score, which can accurately reflect the privacy leakage based on the transformation policy and a semi-trained model which is trained for only a few epochs. 

We first define a metric \texttt{GradSim}, which measures the gradient similarity of two input samples ($x_1$, $x_2$) with the same label $y$:

\begin{myequation}
	\texttt{GradSim}(x_{1}, x_{2}) = \frac{\langle \bigtriangledown W(x_1, y),  \bigtriangledown W(x_2, y) \rangle} 
										{|| \bigtriangledown W(x_1, y) || \cdot || \bigtriangledown W(x_2, y) ||}.
\label{eq:gradsim}
\end{myequation}

Assume the transformed image is $\widehat{x}$, which the adversary tries to reconstruct. He starts from a random input $x'=x_0$, and updates $x'$ iteratively using Equation \ref{eq:attack} until $\bigtriangledown W(x', y)$ approaches $\bigtriangledown W(\widehat{x}, y)$. Figure~\ref{fig:motivation} visualizes this process: the y-axis is the gradient similarity $\texttt{GradSim}(x', \widehat{x})$, and x-axis is $i\in[0, 1]$ such that $x'=(1-i)*x_0 + i*\widehat{x}$. The optimization starts with $i=0$ (i.e., $x=x_0$) and ideally completes at $i=1$ (i.e., $x'=\widehat{x}$ and $\texttt{GradSim}=1$).

\begin{figure} [t]
   \centering
    \begin{tikzpicture}
        \node[anchor=south west,inner sep=0] at (0,0){\includegraphics[ width=0.9\linewidth]{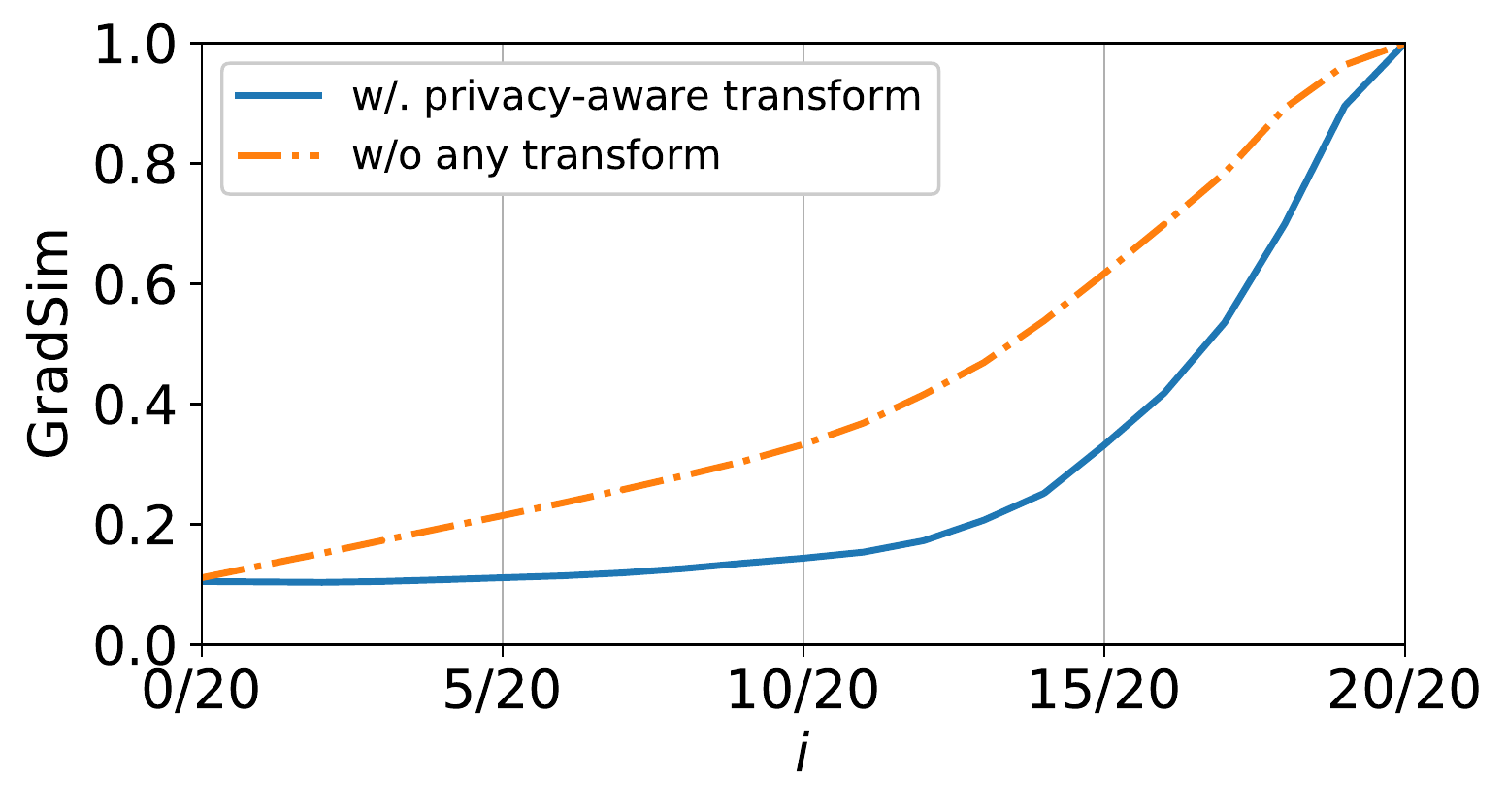}}; 
        \node at (2.1, 1.9){\includegraphics[width=0.08\linewidth]{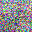}};
        \node at (3.6, 2.2){\includegraphics[width=0.08\linewidth]{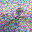}};
        \node at (5.15, 3){\includegraphics[width=0.08\linewidth]{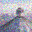}};
        \node at (6.0, 3.5){\includegraphics[width=0.08\linewidth]{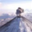}};
        
        \node at (2.9, 1.2){\includegraphics[width=0.08\linewidth]{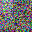}};
        \node at (4.4, 1.3){\includegraphics[width=0.08\linewidth]{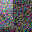}};
        \node at (5.9, 1.7){\includegraphics[width=0.08\linewidth]{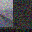}};
        \node at (6.9, 3){\includegraphics[width=0.08\linewidth]{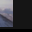}};
    \end{tikzpicture} 
    \vspace{-1em}
	\caption{Visualization of the optimization process in reconstruction attacks.}
	\vspace{-1em}
    \label{fig:motivation}
\end{figure}

A good policy can thwart the convergence from $x_0$ to $\widehat{x}$. As shown in Figure~\ref{fig:motivation} (blue solid line), $\texttt{GradSim}$ is hardly changed with $i$ initially from $x_0$. This reveals the difficulty of the adversary to find the correct direction towards $\widehat{x}$ based on the gradient distance. In contrast, if the collaborative learning system does not employ any transformation function (red dash line), $\texttt{GradSim}$ is increased stably with $i$. This gives the adversary an indication to discover the correct moving direction, and steadily make $x'$ approach $x$ by minimizing the gradient distance.

Based on this observation, we use the area under the $\texttt{GradSim}$ curve to denote the effectiveness of a transformation policy in reducing privacy leakage. A good transformation policy will give a small area as the $\texttt{GradSim}$ curve is flat for most values of $i$ until there is a sharp jump when $i$ is close to 1. In contrast, a leaky learning system has a larger area as the $\texttt{GradSim}$ curve increases gradually with $i$. 
Formally, our privacy score is defined as below:

\begin{myalign}
	& S_{pri}(T) = \frac{1}{|D|} \sum_{x\in D} \int_{0}^{1} \texttt{GradSim}(x'(i), T(x)) di, \nonumber \\
	& x'(i) = (1-i) * x_0 + i * T(x).
\label{eq:privacy-score1}
\end{myalign}

For simplicity, we can approximate this score as a numeric integration, which is adopted in our implementation:


\begin{myalign}
	S_{pri}(T) \approx \frac{1}{|D|K} \sum_{x\in D} \sum_{j=0}^{K-1} \texttt{GradSim}(x'(\frac{j}{K}), T(x)).
\label{eq:privacy-score}
\end{myalign}


\noindent\textbf{Empirical validation.}
We also run some experiments to empirically verify the correlation between $S_{pri}$ and PSNR. Specifically, we randomly select 100 transformation policies, and apply each to the training set. For each policy, we collect the PSNR value by performing the reconstruction attack \cite{geiping2020inverting} with a reduced iteration of 2500. We also measure the privacy score using Equation \ref{eq:privacy-score}. As shown in Figure \ref{fig:privacy-distribution}, $S_{pri}$ is linearly correlated to PSNR with a Pearson Correlation Coefficient of 0.697. This shows that we can use $S_{pri}$ to quantify the attack effects. 


\renewcommand{\multirowsetup}{\centering}
\begin{figure} [t]
    \begin{center}
    \begin{tabular}{c}
        \includegraphics[ width=0.9\linewidth]{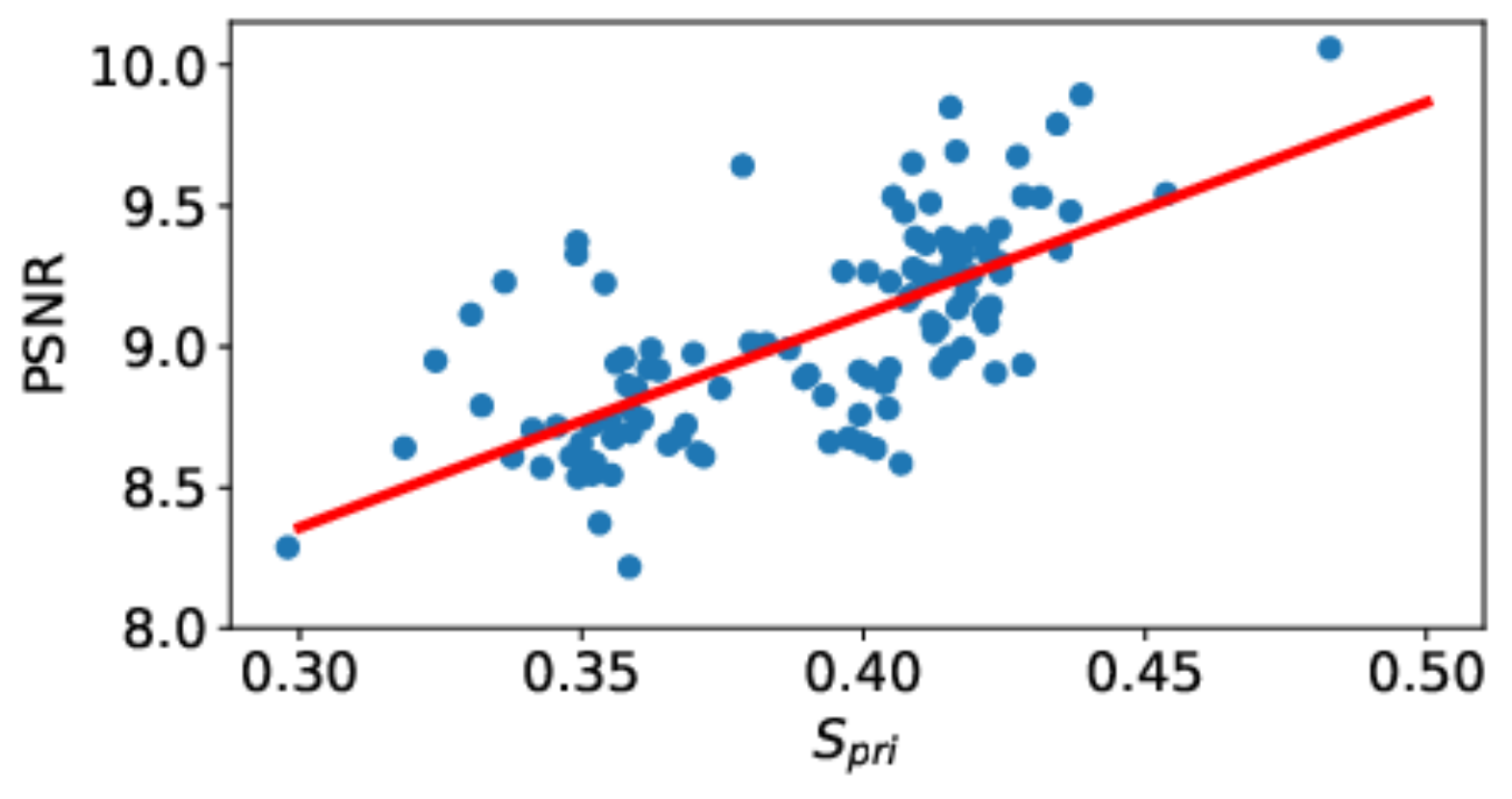}
    \end{tabular}
    \end{center}
    \vspace{-2em}
	\caption{Correlation between PSNR and $S_{pri}$.}
	\vspace{-1em}
    \label{fig:privacy-distribution}
\end{figure}




\subsection{Accuracy Score}
\label{sec:performance-metric}
Another important requirement for a qualified policy is to maintain model accuracy. Certain transformations introduce large-scale perturbations to the samples, which can impair the model performance. We expect to have an efficient and accurate criterion to judge the performance impact of each transformation policy during the search process. 

Joseph \etal~\cite{mellor2020neural} proposed a novel technique to search neural architectures without model training. It empirically evaluates the correlations between the local linear map and the architecture performance, and identifies the maps that yield the best performance. Inspired by this work, we adopt the technique to search for the transformations that can preserve the model performance.


Specifically, we prepare a randomly initialized model $f$, and a mini-batch of data samples transformed by the target policy $T$: $\{\widehat{x}_{n}\}_{n=1}^{N}$. We first calculate the Gradient Jacobian matrix, as shown below: 

\begin{myequation}
J = 
\begin{pmatrix}
\frac{\partial f}{\partial \widehat{x}_{1}}, & \frac{\partial f}{\partial \widehat{x}_{2}}, & \cdots & \frac{\partial f}{\partial \widehat{x}_{N}} \\
\end{pmatrix}^{\top}.
\label{eq:J}
\end{myequation}

Then we compute its correlation matrix:

\begin{myalign}
	& \left(M_J\right)_{ i, j} = \frac{1}{N} \sum_{n=1}^NJ_{i, n}, \nonumber \\
	& C_J = (J-M_J)(J-M_J)^T, \nonumber \\
	& \left(\Sigma_J\right)_{i, j} = \frac{\left(C_J\right)_{i, j}}{\sqrt{\left(C_J\right)_{i, i} \cdot \left(C_J\right)_{j, j} }}.
\label{eq:cor-matrix}
\end{myalign}

Let $\sigma_{J, 1} \leq \dots \leq \sigma_{J, N}$ be the $N$ eigenvalues of $\Sigma_J$. 
Then our accuracy score is given by

\begin{equation}
	S_{acc}(T)=\frac{1}{N}\sum_{i=0}^{N-1} {log(\sigma_{J,i} + \epsilon) + (\sigma_{J,i} + \epsilon)^{-1}},
\label{eq:accuracy-score}
\end{equation}
where $\epsilon$ is set as $10^{-5}$ for numerical stability. This accuracy score can be used to quickly filter out policies which incur unacceptable performance penalty to the model. 


%

\subsection{Searching and Applying Transformations}
\label{sec:search-train}
We utilize the privacy and accuracy scores to identify the optimal policies, and apply them to collaborative training. 

\noindent\textbf{Search space.} We consider the data augmentation library adopted by AutoAugment \cite{iccv2019autoaugment,pool}. This library contains 50 various image transformation functions, including rotation, shift, inversion, contrast, posterization, etc. 
We consider a policy combining at most $k$ functions. This leads to a search space of $\sum_{i=1}^k 50^i$. Instead of iterating all the policies, we only select and evaluate $C_{max}$ policies. For instance, in our implementation, we choose $k=3$, and the search space is $127,550$. We set $C_{max}=1,500$, which is large enough to identify qualified policies. 

\noindent\textbf{Search algorithm.} 
Various AutoML methods have been designed to search for the optimal architecture, 
e.g., importance sampling~\cite{neal2001annealed}, evolutionary sampling~\cite{kwok2005evolutionary}, reinforcement learning-based sampling~\cite{zoph2016neural}. We adopt a simple \emph{random} search strategy, which is efficient and effective in discovering the optimal policies. 

Algorithm \ref{alg:search} illustrates our search process. Specifically, we aim to identify a policy set $\mathcal{T}$ with $n$ qualified policies. We need to prepare two local models: (1) $M^{s}$ is used for privacy quantification. It is trained only with 10\% of the original training set for 50 epochs. This overhead is equivalent to the training with the entire set for 5 epochs, which is very small. (2) $M^{r}$ is a randomly initialized model without any optimization, which is used for accuracy quantification. 
We randomly sample $C_{max}$ policies, and calculate the privacy and accuracy scores of each policy. The policies with accuracy scores lower than a threshold $T_{acc}$ will be filtered out. We select the top-$n$ policies based on the privacy score to form the final policy set $\mathcal{T}$. 



\begin{algorithm}[htb]
	\caption{Searching optimal transformations.}\label{alg:search} 
	\SetKwInOut{Input}{Input}
	\SetKwInOut{Output}{Output} 
	    \Input{Augmentation library $\mathcal{P}$, $T_{acc}$, $C_{max}$, $M^s$, $M^r$, $D$}
		\Output{Optimal policy set $\mathcal{T}$ with $n$ policies}
        
		\For{$i \in [1, C_{max}]$ \label{line:init}}{
		    Sample functions from $\mathcal{P}$ to form a policy $T$\;
		    Calculate $S_{acc}(T)$ from $M^r$, $D$ (Eq. \ref{eq:accuracy-score})\;
		    \If{$S_{acc}(T) \geq T_{acc}$}{
		        \eIf{$|\mathcal{T}|<n$}{
		            Insert $T$ to $\mathcal{T}$\;
		        }{
		            Calculate $S_{pri}(T)$ from $M^s$, $D$ (Eq. \ref{eq:privacy-score})\;
		            $T^* \gets \text{argmax}_{T'\in \mathcal{T}}S_{pri}(T')$\;
		            \If{$S_{pri}(T)<S_{pri}(T^*)$}{
		                Replace $T^*$ with $T$ in $\mathcal{T}$\;
		            }
		        }
		    }
		}
		\If{$|\mathcal{T}|<n$}{
            Go to Line \ref{line:init}\;	
		}
		\Return $\mathcal{T}$
%
\end{algorithm}


\noindent\textbf{Applying transformations.}
With the identified policy set $\mathcal{T}$, we can apply the functions over the sensitive training data. One possible solution is to always pick the policy with the smallest $S_{pri}$, and apply it to each sample. However, a single fixed policy can incur domain shifts and bias in the input samples. This can impair the model performance although we have tested it with the accuracy metric. 

Instead, we can adopt a hybrid augmentation strategy which is also used in \cite{iccv2019autoaugment}: we randomly select a transformation policy from $\mathcal{T}$ to preprocess each data sample. All the selected transformation policies cannot have common transformation functions. This can guarantee low privacy leakage and high model accuracy. Besides, it can also improve the model generalization and eliminate domain shifts.

\section{Experiments}
\label{sec:exp}
\renewcommand{\multirowsetup}{\centering}

\subsection{Implementation and Configurations}
\noindent\textbf{Datasets and models.} Our approach is applicable to various image datasets and classification models. Without loss of generality, we choose two datasets (CIFAR100~\cite{krizhevsky2009learning}, Fashion MNIST~\cite{xiao2017fashion}) and two conventional DNN models (ResNet20~\cite{he2016deep}, 8-layer ConvNet). These were the main targets of reconstruction attacks in prior works.

\noindent\textbf{System and attack implementation.}
We implement a collaborative learning system with ten participants, where each one owns a same number of training samples from the same distribution. They adopt the SGD optimizer with momentum, weight decay and learning decay techniques to guarantee the convergence of the global model. 


Our solution is able to thwart all existing reconstruction attacks with their variants. We evaluate six attacks in our experiments, named in the format of ``optimizer+distance measure''. These techniques\footnote{The attack in \cite{zhao2020idlg} inherited the same technique from \cite{zhu2019deep}, with a smaller computational cost. So we do not consider it in our experiments.} cover different optimizers and distance measures: (1) LBFGS+L2 \cite{zhu2019deep}; (2) Adam+Cosine \cite{geiping2020inverting}; (3) LBFGS+Cosine; (4) Adam+L1; (5) Adam+L2; (6) SGD+Cosine. It is straightforward that the reconstruction attacks become harder with larger batch sizes. To fairly evaluate the defenses, we consider the strongest attacks where the batch size is 1. 

\noindent\textbf{Defense implementation.}
We adopt the data augmentation library \cite{pool}, which contains 50 various transformations. 
We consider a policy with maximum 3 functions concatenated. It is denoted as $i-j-k$, where $i$, $j$, and $k$ are the function indexes from \cite{pool}. Note that 
index values can be the same, indicating the same function is applied multiple times. 


We implement the following defenses as the baseline. 

\begin{packeditemize}
    \item \emph{Gaussian/Laplacian}: using differential privacy to obfuscate the gradients with Gaussian or Laplacian noise. For instance, Gaussian($10^{-3}$) suggests a noise scale of $N(0, 10^{-3})$.
    \item \emph{Pruning}: adopting the layer-wise pruning technique~\cite{dutta2019discrepancy} to drop parameter gradients whose absolute values are small. For instance, a compression ratio of 70\% means for each layer, the gradients are set as zero if their absolute values rank after the top-$30\%$. 
    \item \emph{Random augmentation}: we randomly sample transformation functions from \cite{pool} to form a policy. For each experiment, we apply 10 different random policies to obtain the average results.  
\end{packeditemize}


We adopt PSNR to measure the visual similarity between the attacker's reconstructed samples and transformed samples, as the attack effects. 
We measure the trained model's accuracy over the corresponding validation dataset to denote the model performance. 

\noindent\textbf{Testbed configuration.} We adopt PyTorch framework~\cite{paszke2019pytorch} to realize all the implementations. All our experiments are conducted on a server equipped with one NVIDIA Tesla V100 GPU and 2.2GHz Intel CPUs. 

\subsection{Search and Training Overhead}


\noindent\textbf{Search cost.}
For each transformation policy under evaluation, we calculate the average $S_{pri}$ of 100 images randomly sampled from the validation set\footnote{The first 100 images in the validation set are used for attack evaluation, not for $S_{pri}$ calculation.}. We also calculate $S_{acc}$ with 10 forward-background rounds. We run 10 search jobs in parallel on one GPU. Each policy can be evaluated within 1 minutes. Evaluation of all $C_{max}=1,500$ policies can be completed within 2.5 hours. The entire search overhead is very low. In contrast, the attack time of reconstructing 100 images using \cite{geiping2020inverting} is about 10 GPU hours.

\noindent\textbf{Training cost.}
Applying the searched policies to the training samples can be conducted offline. So we focus on the online training performance. We train the ResNet20 model on CIFAR100 with 200 epochs. Figure \ref{fig:training-analysis} reports the accuracy and loss over the training and validation sets with and without our transformation policies. We can observe that although the transformation policies can slightly slow down the convergence speed on the training set, the speeds on the validation set are identical. This indicates the transformations incur negligible overhead to the training process.

\begin{figure}
    \centering
    \begin{subfigure}[b]{0.49\linewidth}
        \centering
        \includegraphics[ width=\linewidth]{ 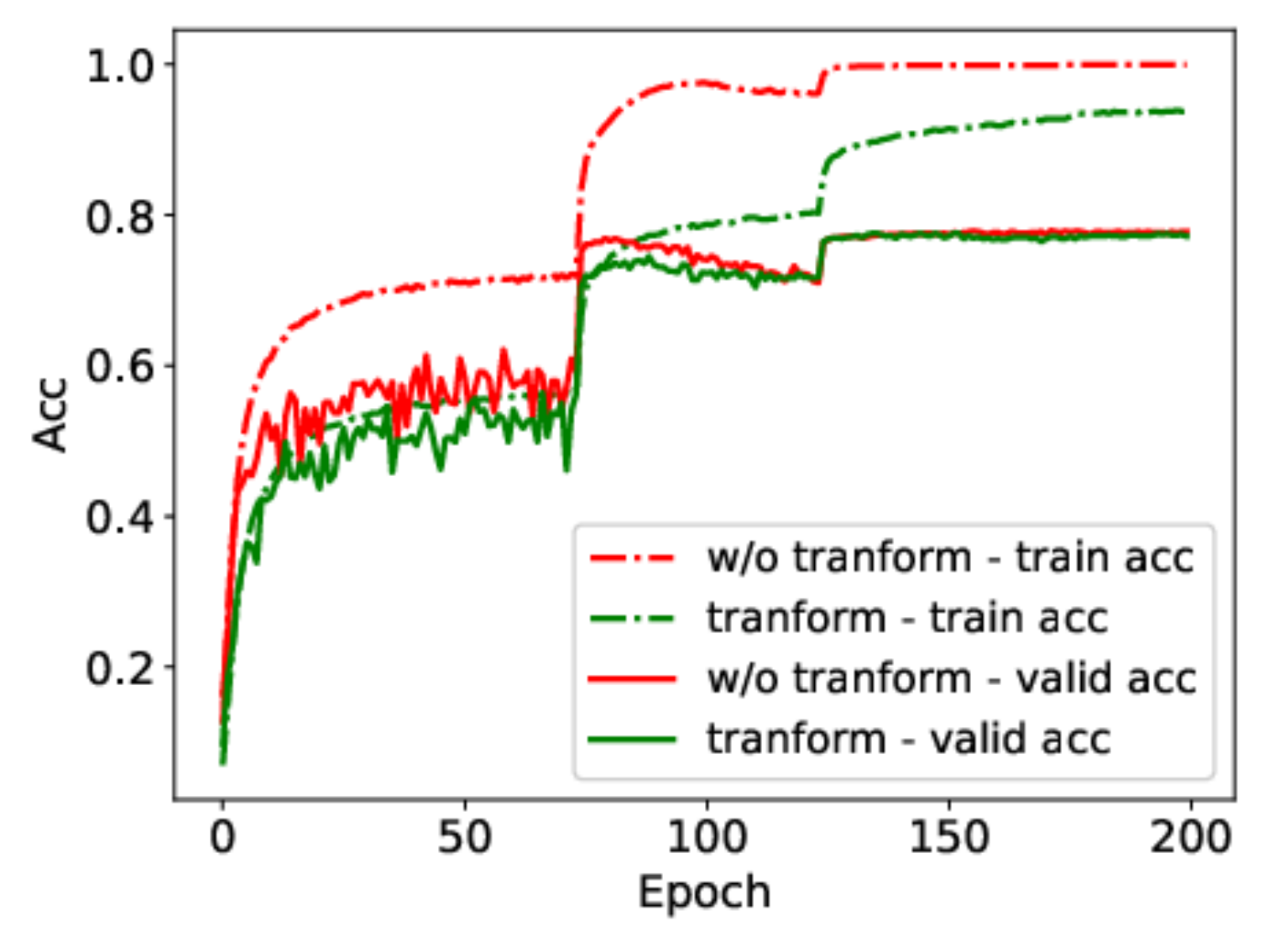}
        \caption{Accuracy}
        \label{fig:accuracy}
     \end{subfigure}    
    \begin{subfigure}[b]{0.49\linewidth}
        \centering
        \includegraphics[ width=\linewidth]{ 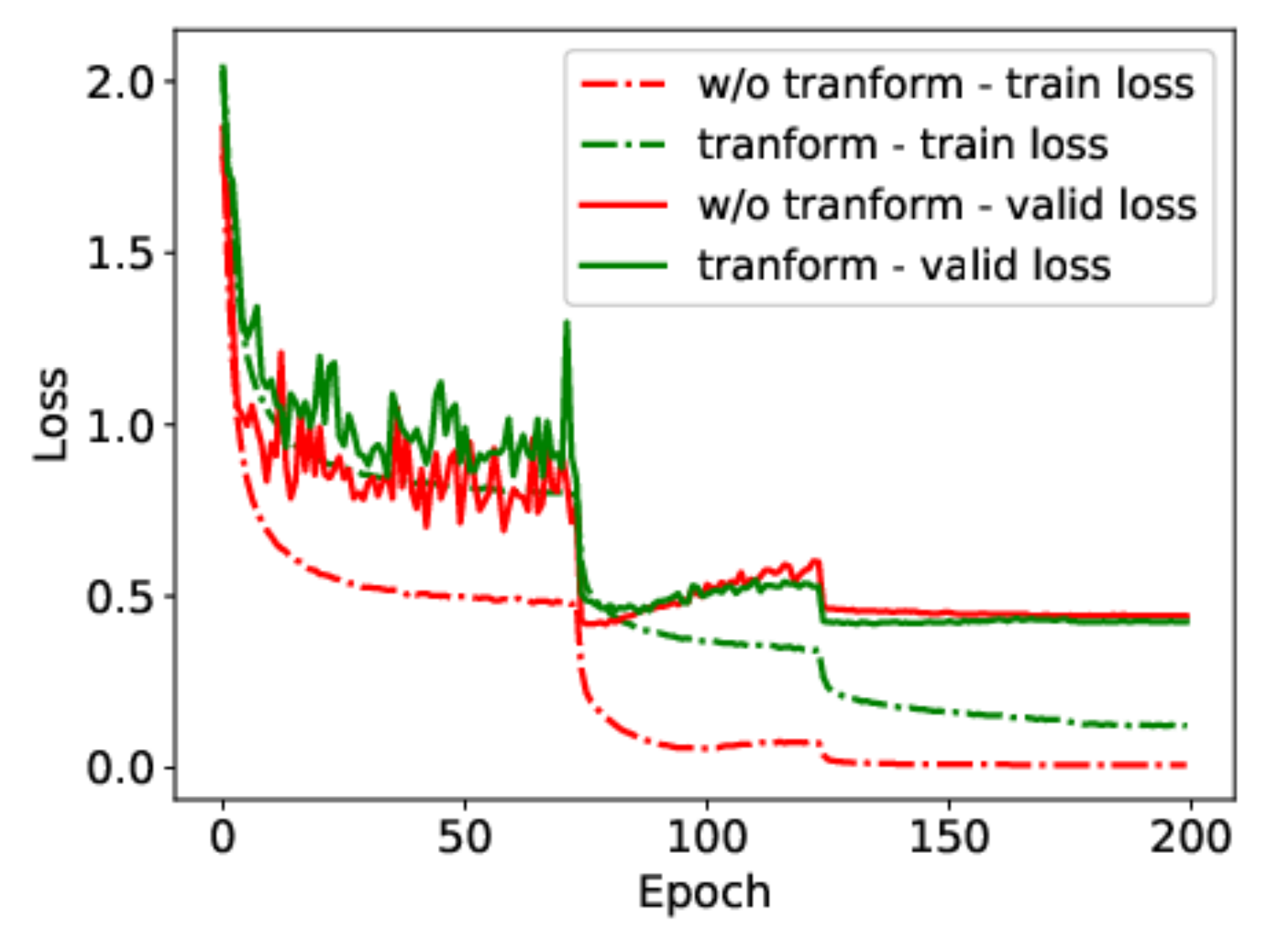}
        \caption{Loss}
        \label{fig:loss}
     \end{subfigure}    
    \caption{Model performance of ResNet20 on CIFAR100 during the training process.} 
    \label{fig:training-analysis}
\end{figure}



\subsection{Effectiveness of the Searched Policies}

As an example, Figure~\ref{fig:overview} illustrates the visual comparison of the reconstructed images with and without the searched policies under the Adam+Cosine attack \cite{geiping2020inverting} for the two datasets. We observe that without any transformations, the adversary can recover the images with very high fidelity (row 2). In contrast, after the training samples are transformed (row 3), the adversary can hardly obtain any meaningful information from the recovered images (row 4). We have similar results for other attacks as well. 

\begin{figure*} [t]
    \begin{center}
    \begin{tabular}{c@{\hspace{0.01\linewidth}}c@{\hspace{0.01\linewidth}}c@{\hspace{0.03\linewidth}} c@{\hspace{0.01\linewidth}}c@{\hspace{0.01\linewidth}}c@{\hspace{0.03\linewidth}} c@{\hspace{0.01\linewidth}}c@{\hspace{0.01\linewidth}}c@{\hspace{0.03\linewidth}} c@{\hspace{0.01\linewidth}}c@{\hspace{0.01\linewidth}}c@{\hspace{0.01\linewidth}}   }
        \includegraphics[ width=0.07\linewidth]{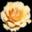} &
        \includegraphics[ width=0.07\linewidth]{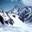} &
        \includegraphics[ width=0.07\linewidth]{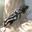} &
        \includegraphics[ width=0.07\linewidth]{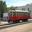} &
        \includegraphics[ width=0.07\linewidth]{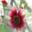} &
        \includegraphics[ width=0.07\linewidth]{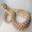} &
        \includegraphics[ width=0.07\linewidth]{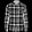} &
        \includegraphics[ width=0.07\linewidth]{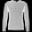} &
        \includegraphics[ width=0.07\linewidth]{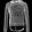} &
        \includegraphics[ width=0.07\linewidth]{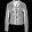} &
        \includegraphics[ width=0.07\linewidth]{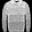} &
        \includegraphics[ width=0.07\linewidth]{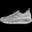}  \\
        
        \includegraphics[ width=0.07\linewidth]{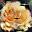} &
        \includegraphics[ width=0.07\linewidth]{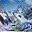} &
        \includegraphics[ width=0.07\linewidth]{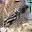} &
        \includegraphics[ width=0.07\linewidth]{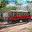} &
        \includegraphics[ width=0.07\linewidth]{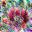} &
        \includegraphics[ width=0.07\linewidth]{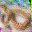} &
        \includegraphics[ width=0.07\linewidth]{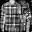}& 
        \includegraphics[ width=0.07\linewidth]{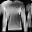} &
        \includegraphics[ width=0.07\linewidth]{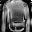} &  
        \includegraphics[ width=0.07\linewidth]{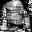}  &
        \includegraphics[ width=0.07\linewidth]{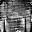} &
        \includegraphics[ width=0.07\linewidth]{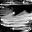} \\
        
        13.97dB &
        12.14dB &
        16.32dB &
        15.61dB &
        11.96dB &
        14.83dB &
        10.42dB &
        9.65dB &
        13.57dB &
        9.52dB &
        9.49dB &
        9.17dB \\
        
        \includegraphics[ width=0.07\linewidth]{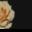} & 
        \includegraphics[ width=0.07\linewidth]{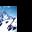} &
        \includegraphics[ width=0.07\linewidth]{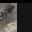} &
        \includegraphics[ width=0.07\linewidth]{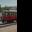} &
        \includegraphics[ width=0.07\linewidth]{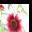} &
        \includegraphics[ width=0.07\linewidth]{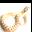} &
        \includegraphics[ width=0.07\linewidth]{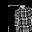} &
        \includegraphics[ width=0.07\linewidth]{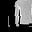} &
        \includegraphics[ width=0.07\linewidth]{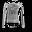} &
        \includegraphics[ width=0.07\linewidth]{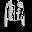} &
        \includegraphics[ width=0.07\linewidth]{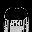} &
        \includegraphics[ width=0.07\linewidth]{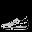} \\
        
        \includegraphics[ width=0.07\linewidth]{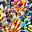} & 
        \includegraphics[ width=0.07\linewidth]{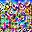} &
        \includegraphics[ width=0.07\linewidth]{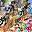} &
        \includegraphics[ width=0.07\linewidth]{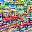}&
        \includegraphics[ width=0.07\linewidth]{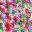} &
         \includegraphics[ width=0.07\linewidth]{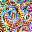} &
        \includegraphics[ width=0.07\linewidth]{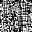} &
        \includegraphics[ width=0.07\linewidth]{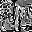} &
        \includegraphics[ width=0.07\linewidth]{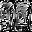} &
        \includegraphics[ width=0.07\linewidth]{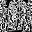} &
        \includegraphics[ width=0.07\linewidth]{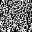} &
        \includegraphics[ width=0.07\linewidth]{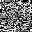}  \\
        
        6.94dB &
        5.62dB &
        7.64dB &
        6.97dB &
        6.74dB &
        7.02dB &
        7.72dB &
        6.56dB &
        8.24dB &
        6.63dB &
        7.02dB &
        6.04dB \\
        \\
        
        \multicolumn{3}{l}{(a) CIFAR100 with ResNet20} & \multicolumn{3}{l}{(b) CIFAR100 with ConvNet} & \multicolumn{3}{l}{(c) F-MNIST with ResNet20} & \multicolumn{3}{l}{(d) F-MNIST with ConvNet} \\
    \end{tabular}
    \end{center}
    \vspace{-1em}
    \caption{Visual results and the PSNR values of the reconstruction attacks \cite{geiping2020inverting} with and without our defense. Row 1: clean samples. Row 2: reconstructed samples without transformation. Row 3: transformed samples. Row 4: reconstructed samples with transformation. The adopted transformations are the corresponding \emph{Hybrid} policies in Table \ref{tab:overview}. }
    \label{fig:overview}
\end{figure*}

Table~\ref{tab:overview} reports the quantitative results of Adam+Cosine attacks and model accuracy. For each dataset and architecture, we consider the model training with no transformations,
random selected policies, the top-2 of the searched policies and their hybrid. We observe that randomly selected policies fail to invalidate reconstruction attacks. In contrast, the searched policies can effectively reduce the deep leakage from the gradients. The hybrid of policies exhibits higher generalization ability on the final model. 

\begin{table*}
    \centering
    \captionsetup[subtable]{position = below}
           \resizebox{0.21\textwidth}{!}{
           \begin{subtable}{0.22\linewidth}
               \centering
               \begin{tabular}{c|cc}
                   \hline
                   \textbf{Policy} & \textbf{PSNR} & \textbf{ACC} \\ \hline
                        None & 13.88 &	76.88 \\
                        Random  &	11.41 &	73.94 \\
                        3-1-7 &	6.579 &	70.56 \\
                        43-18-18 & 8.56 & 77.27  \\
                        Hybrid  & 7.64 & 77.92 \\ \hline
               \end{tabular}
               \caption{CIFAR100 with ResNet20}
               \label{tab:overview-cifar1}
           \end{subtable}}%
           \hspace*{2em}
           \resizebox{0.21\textwidth}{!}{
           \begin{subtable}{0.22\linewidth}
               \centering
               \begin{tabular}{c|cc}
                   \hline
                   \textbf{Policy} & \textbf{PSNR} & \textbf{ACC} \\ \hline
                        None & 13.07 &	70.13 \\
                        Random & 12.18 &	69.91 \\
                        21-13-3 & 5.76  &	66.98 \\
                        7-4-15 & 7.75  & 69.67 \\
                        Hybrid & 6.83  & 70.27 \\ \hline
               \end{tabular}
               \caption{CIFAR100 with ConvNet}
               \label{tab:overview-cifar2}
           \end{subtable}}%
           \hspace*{2em}
           \resizebox{0.21\textwidth}{!}{
           \begin{subtable}{0.22\linewidth}
               \centering
               \begin{tabular}{c|cc}
                   \hline
                   \textbf{Policy} & \textbf{PSNR} & \textbf{ACC} \\ \hline
                        None & 10.04 &	95.03 \\
                        Random &	9.23 &	91.16 \\
                        19-15-45&	7.01  & 91.33 \\
                        2-43-21 &   7.75  & 89.41 \\
                        Hybrid &   7.60  & 92.23 \\ \hline
               \end{tabular}
               \caption{F-MNIST with ResNet20}
               \label{tab:overview-mnist1}
           \end{subtable}}%
           \hspace*{2em}
           \resizebox{0.21\textwidth}{!}{
           \begin{subtable}{0.22\linewidth}
               \centering
               \begin{tabular}{c|cc}
                   \hline
                   \textbf{Policy} & \textbf{PSNR} & \textbf{ACC} \\ \hline
                        None &	9.12 &	94.25 \\
                        Random &	8.83 &	90.18 \\
                        42-28-42 &	7.01  & 91.33 \\
                        14-48-48 &   6.75  & 90.56 \\
                        Hybrid &   6.94  & 91.35 \\ \hline
               \end{tabular}
               \caption{F-MNIST with ConvNet}
               \label{tab:overview-mnist2}
           \end{subtable}}%
            \captionsetup[table]{position=bottom}
           \caption{PSNR (db) and model accuracy (\%) of different transformation configurations for each architecture and dataset.}
           \label{tab:overview}

\end{table*}

Table \ref{tab:attack} reports the PSNR values of the hybrid strategy against different reconstruction attacks and their variants. Compared with the training process without any defenses, the hybrid of searched transformations can significantly reduce the image quality of the reconstructed images, and eliminate information leakage in different attacks. 
  

\begin{table}[t]
    \begin{center}
    \small
    \resizebox{0.47\textwidth}{!}{
    \begin{tabular}{c c c|| c c c }
        \toprule
          \textbf{Attack} &  \textbf{None} & \textbf{Hybrid} & \textbf{Attack} &  \textbf{None} & \textbf{Hybrid} \\
        \midrule 
        LBFGS+L2  & 6.93 & 4.79 & LBFGS+COS  & 10.33 & 6.16 \\
        \midrule
        Adam+Cosine & 13.88 & 7.64 & Adam+L2 & 10.55 & 7.61 \\
        \midrule
        Adam+L1  & 9.99 & 6.97 & SGD+COS & 14.04 & 7.71 \\
        \bottomrule
    \end{tabular}}
    \end{center}
    \caption{The PSNR values (db) between the reconstructed and transformed images under different attack techniques.}
    \label{tab:attack}
\end{table}

\noindent\textbf{Comparisons with other defenses.}
We also compare our solution with state-of-the-art privacy-preserving methods proposed in prior works. We consider model pruning with different compression ratios, and differential privacy with different noise scales and types. Table~\ref{tab:defense} illustrates the comparison results. We observe that these solutions can hardly reduce the PSNR values, and the model accuracy is decreased significantly with larger perturbation. These results are consistent with the conclusion in \cite{zhu2019deep}. In contrast, our solution can significantly destruct the quality of recovered images, while maintaining high model accuracy.

\begin{table}[t]
    \begin{center}
    \resizebox{0.27\textwidth}{!}{
    \begin{tabular}{c c c c  }
        \toprule
        \textbf{Defense} & \textbf{PSNR} & \textbf{ACC}    \\
        \midrule
        Pruning ($70\%$) & 12.00 & 77.12 \\
        Pruning ($95\%$) & 10.07 & 70.12 \\
        Pruning ($99\%$) & 10.93 & 58.33 \\
        \midrule
        Laplacian ($10^{-3}$) & 11.85 & 74.12 \\
        Laplacian ($10^{-2}$) & 9.67 & 39.59 \\
        Gaussian ($10^{-3}$) & 12.71 & 75.67 \\
        Gaussian ($10^{-2}$) & 11.44 & 48.2 \\
        \midrule
        Hybrid &	7.64 & 77.92 \\
        \bottomrule
    \end{tabular}}
    \end{center}
    \vspace{-1em}
    \caption{Comparisons with existing defense methods under the Adam+Cosine attack.}
    \label{tab:defense}
    \vspace{-1em}
\end{table}

\noindent\textbf{Transferability.}
In the above experiments, we search the optimal policies for each dataset. Actually the searched transformations have high transferability across different datasets. To verify this, we apply the policies searched from CIFAR100 to the tasks of F-MNIST, and Table \ref{tab:db-transfer} illustrates the PSNR and accuracy values. We observe that although these transferred policies are slightly worse than the ones directly searched from F-MNIST, they are still very effective in preserving the privacy and model performance, and much better than the randomly selected policies. This transferability property makes our solution more efficient. 

\begin{table}
    \centering
    \captionsetup[subtable]{position = below}
    \resizebox{0.22\textwidth}{!}{
           \begin{subtable}{0.5\linewidth}
               \centering
               \begin{tabular}{c|cc}
                   \hline
                   \textbf{Policy} & \textbf{PSNR} & \textbf{ACC} \\ \hline
                        None & 10.04 &	95.03 \\
                        3-1-7 &	7.5 &	87.95 \\
                        43-18-18 & 8.13 & 91.29  \\
                        Hybrid  & 8.14 & 91.49 \\ \hline
               \end{tabular}
               \caption{F-MNIST with ResNet20}
               \label{tab:db-transfer1}
           \end{subtable}}
          \hspace*{1em}
          \resizebox{0.22\textwidth}{!}{
           \begin{subtable}{0.5\linewidth}
               \centering
               \begin{tabular}{c|cc}
                   \hline
                   \textbf{Policy} & \textbf{PSNR} & \textbf{ACC} \\ \hline
                        None &	9.12 &	94.25 \\
                        21-13-3 & 7.51  &	74.81 \\
                        7-4-15 & 7.68  & 88.29 \\
                        Hybrid & 7.11  & 87.51 \\ \hline
               \end{tabular}
               \caption{F-MNIST with ConvNet}
               \label{tab:db-transfer2}
           \end{subtable}}
            \captionsetup[table]{position=bottom}
           \caption{Transferability results: applying the same policies from CIFAR100 to F-MNIST.}
           \label{tab:db-transfer}

\end{table}



\subsection{Explanations about the Transformation Effects}
\label{sec:delve-into-augmentation}

In this section, we further analyze the mechanisms of the transformations that can invalidate reconstruction attacks. We first investigate which kinds of transformations are particularly effective in obfuscating input samples. Figure~\ref{fig:transform-dist} shows the privacy score of each transformation. The five transformations with the lowest scores are (red bars in the figure): 3rd [horizontal shifting, 9], 15th [brightness, 9], 18th [contrast, 7], 26th [brightness, 6] and 1st [contrast, 6]; where the parameters inside the brackets are the magnitudes of the transformations. These functions are commonly selected in the optimal policies. 

Horizontal shifting achieves the lowest score, as it incurs a portion of black area, which can undermine the quality of the recovered image during the optimization. Contrast and brightness aim to modify the lightness of an image. These operations can blur the local details, which also increase the difficulty of image reconstruction. Overall, the selected privacy-preserving transformations can distort the details of the images, while maintaining the semantic information.

\begin{figure}
    \centering
    \vspace{-1em}
    \includegraphics[ width=0.85\linewidth]{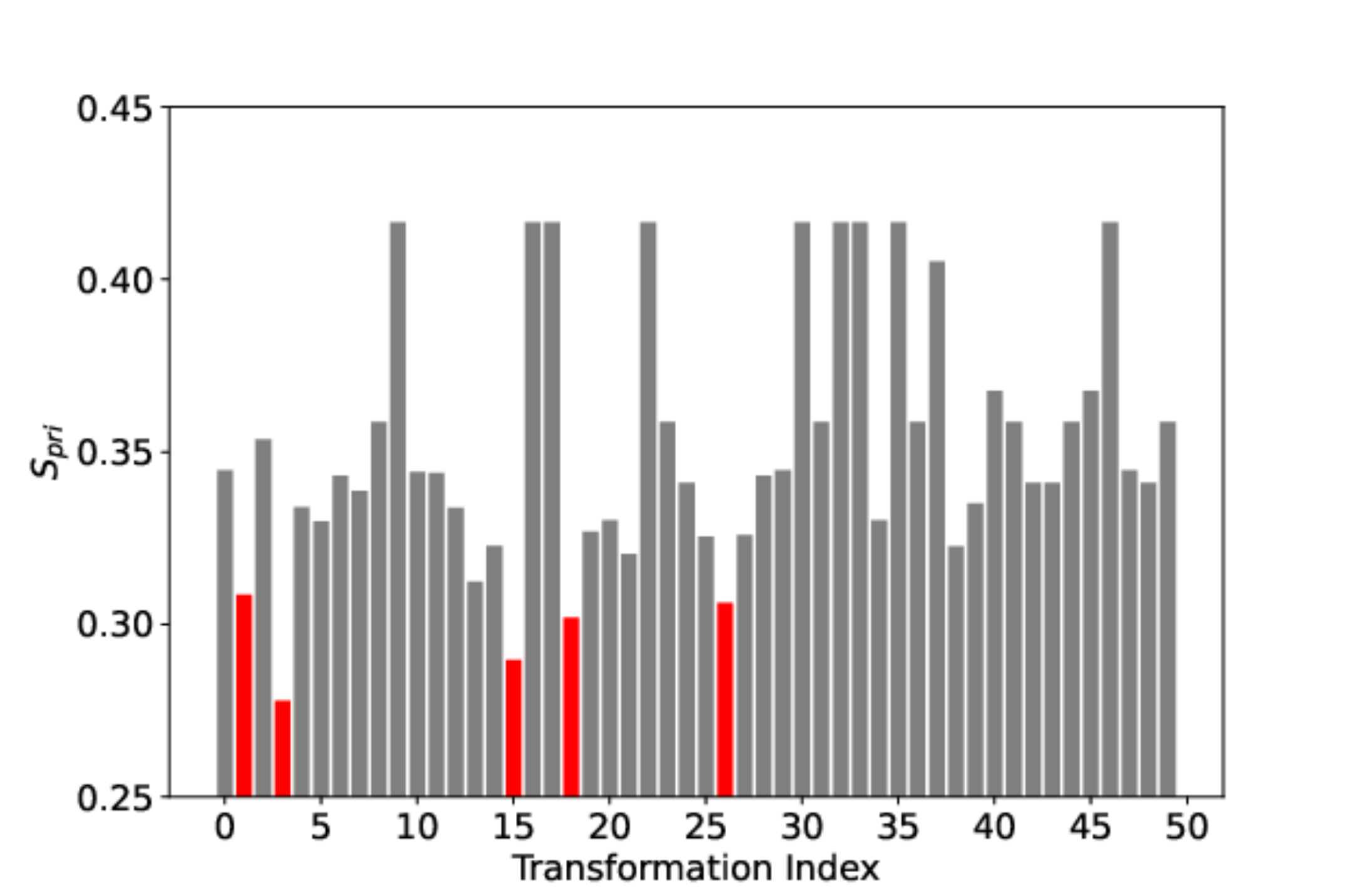}
    \caption{Privacy scores of the 50 transformation functions in the augmentation library.}
    \vspace{-1em}
    \label{fig:transform-dist}
\end{figure}

Next, we explore the attack effects at different network layers. We compare three strategies: (1) no transformation; (2) random transformation policy; (3) searched transformation policy. Figure \ref{fig:deep-shallow} demonstrates the similarity between the gradient of the reconstructed samples and the actual gradient for two shallow layers (a) and two deep layers (b). We can observe that at shallow layers, the similarity scores converge to 0.7 when no or random policy is applied. In contrast, the similarity score stays at lower values when the optimal policy is used. This indicates that the optimal policy makes it difficult to reconstruct low-level visual features of the input,  e.g. color, shape, and texture. The similarity scores for all the three cases are almost the same at deep layers. This reveals the optimal policy has negligible impact on the semantic information of the images used for classification, and the model performance is thus maintained.

\begin{figure}
    \centering
    \begin{subfigure}[b]{\linewidth}
        \centering
        \includegraphics[ width=0.49\linewidth]{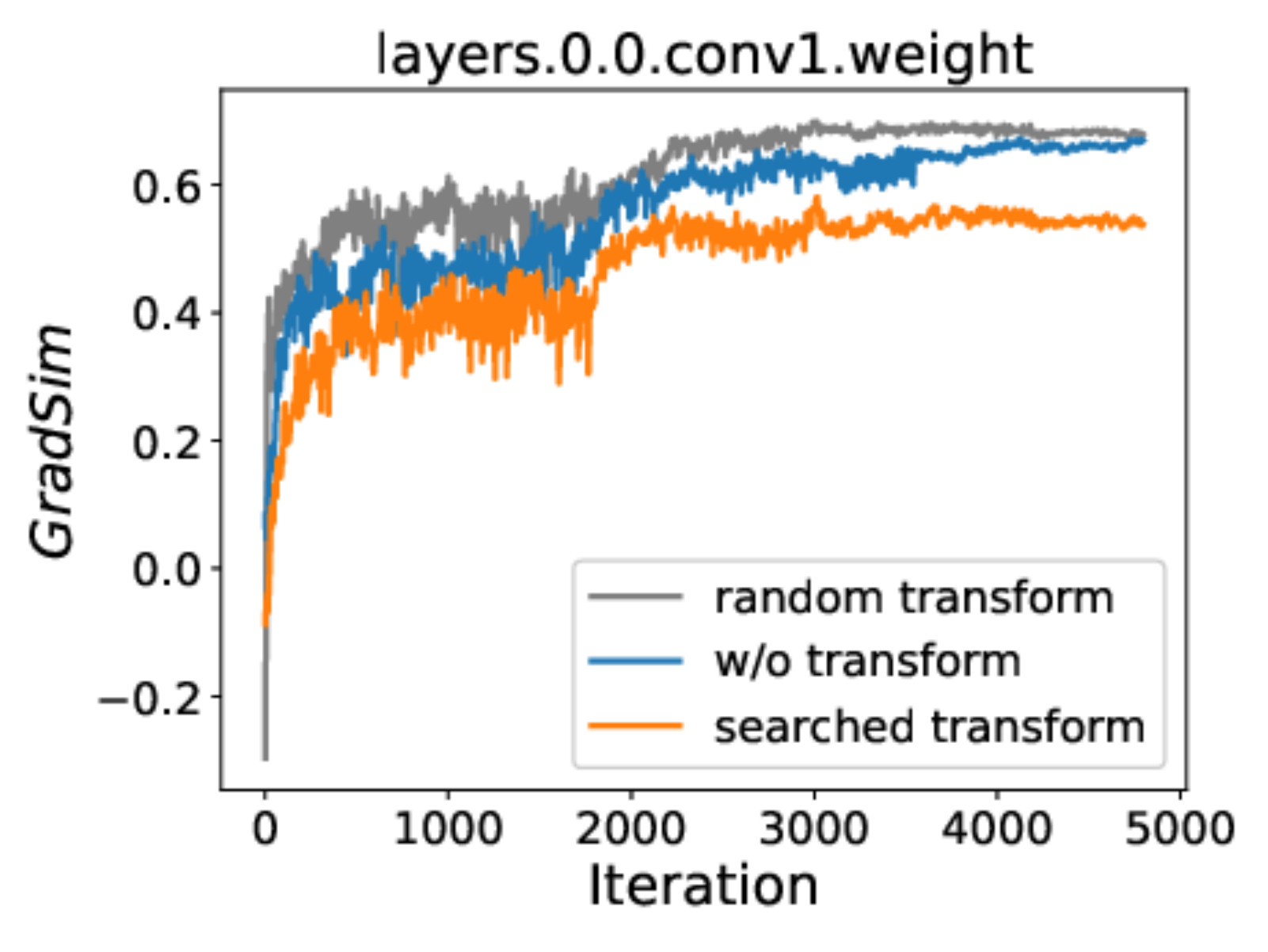}
        \includegraphics[ width=0.49\linewidth]{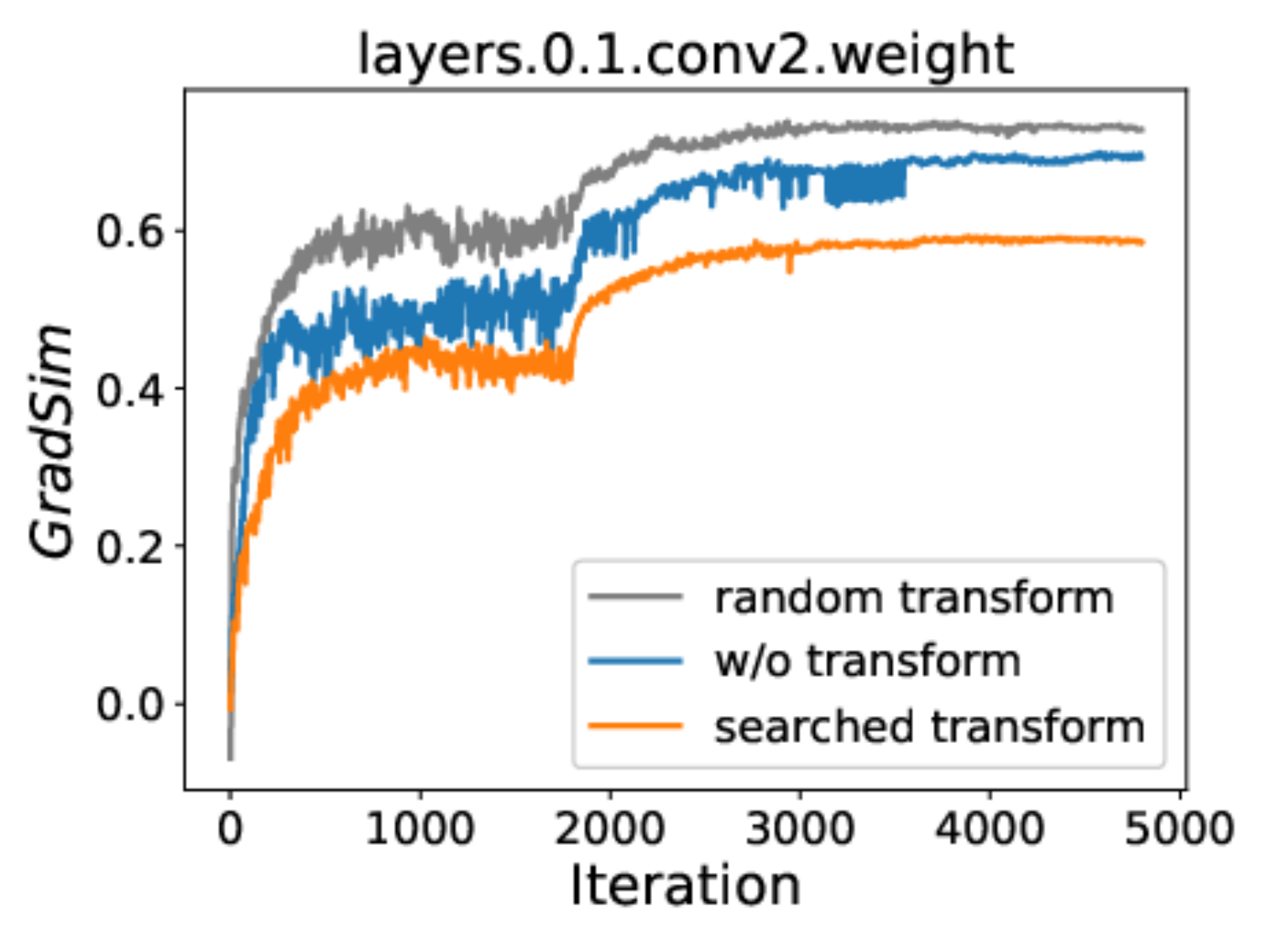}
        \caption{Shallow layers}
        \label{fig:shallow}
     \end{subfigure}    

    \begin{subfigure}[b]{\linewidth}
        \centering
        \includegraphics[ width=0.49\linewidth]{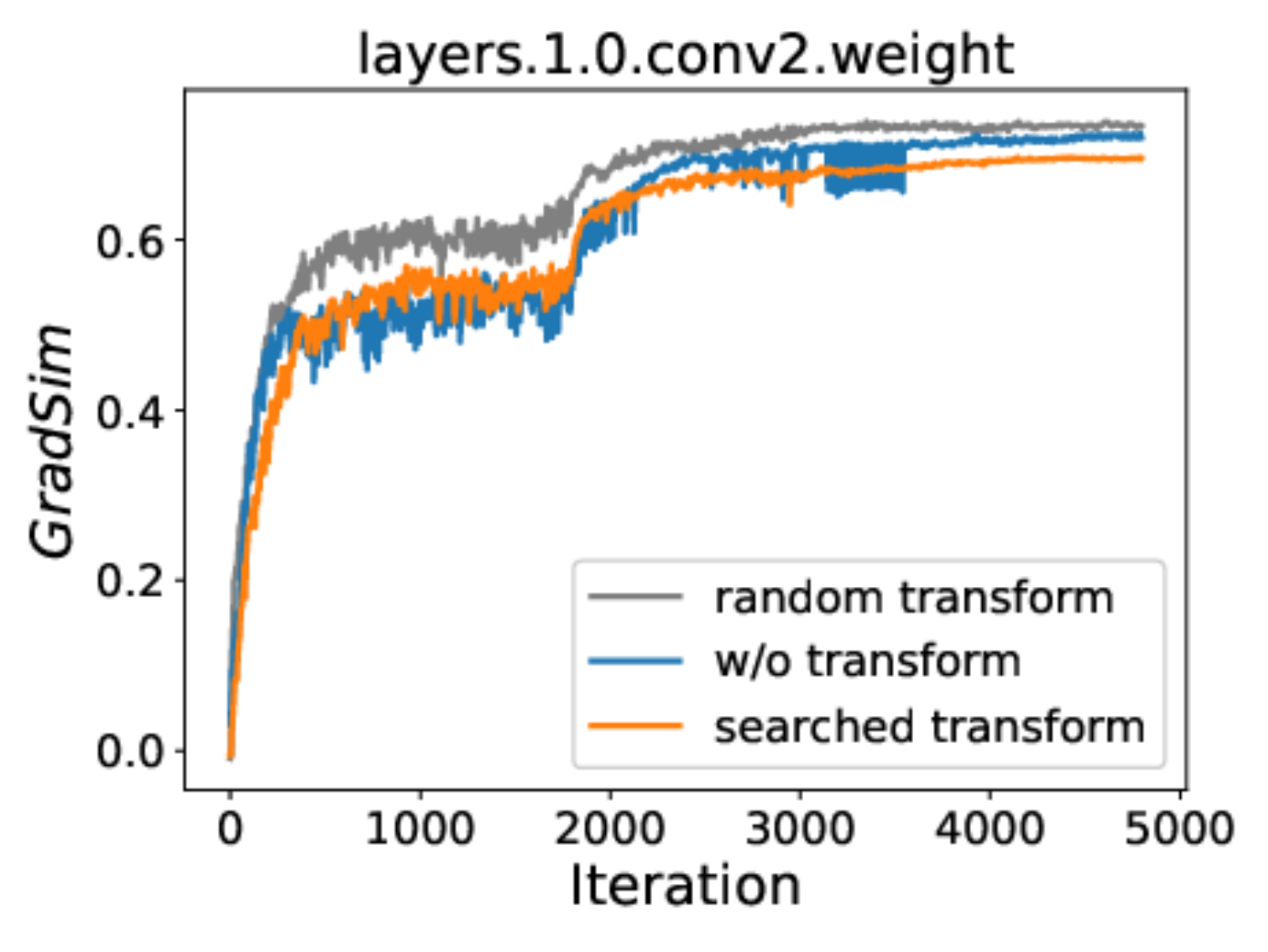}
        \includegraphics[ width=0.49\linewidth]{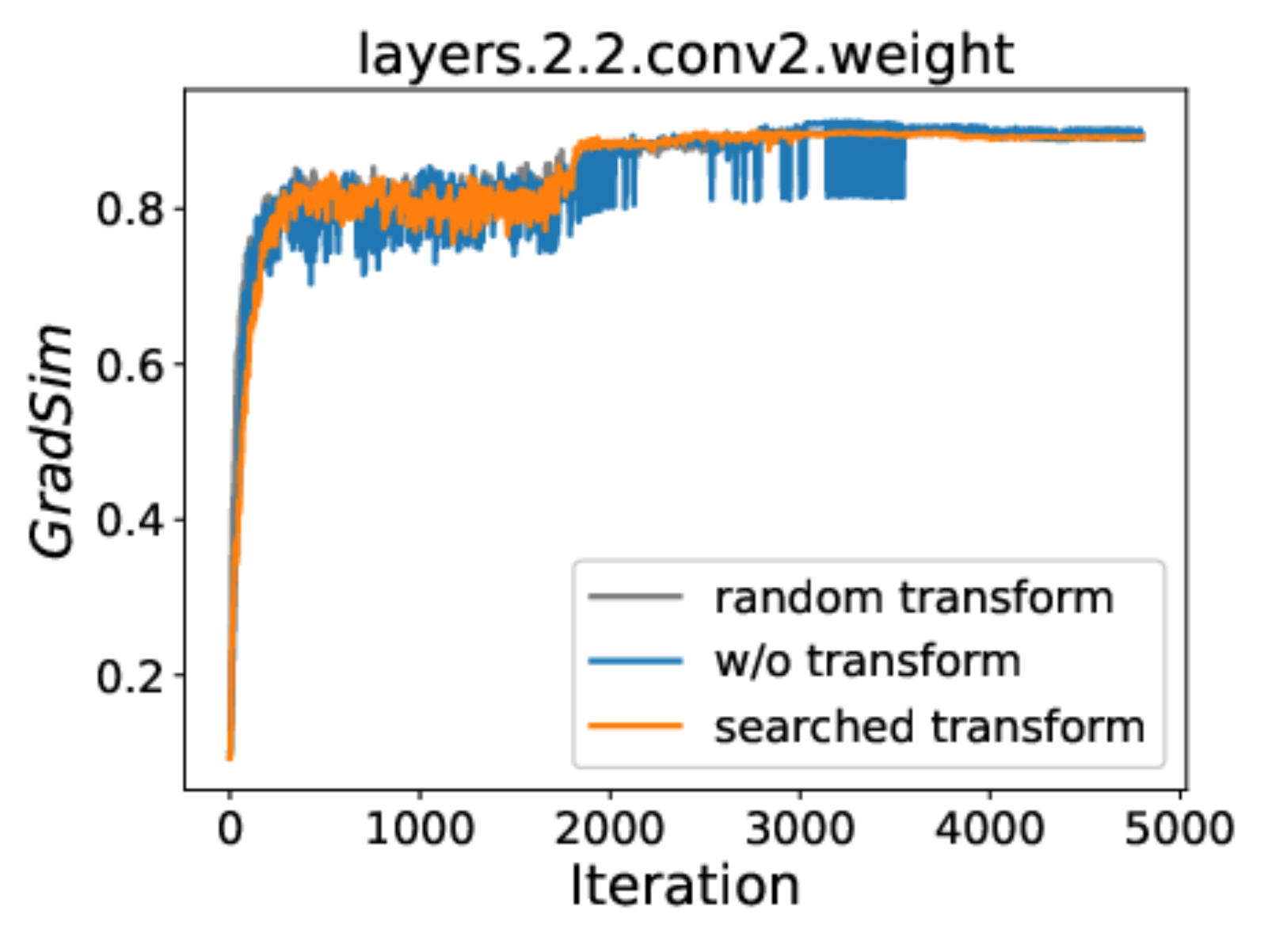}
        \caption{Deep layers}
        \label{fig:Deep}
     \end{subfigure}    
    \caption{Gradient similarity during the reconstruction optimization process, for CIFAR100 with ResNet20. }
    \vspace{-1em}
    \label{fig:deep-shallow}
\end{figure}

\section{Discussions and Future Work}


%



\bheading{Adaptive attack.} Our solution prevents image reconstruction via data augmentation techniques. Although the evaluations show it is effective against existing attacks, a more sophisticated adversary may try to bypass our defense from two aspects. First, instead of starting from a randomly initialized image, he may guess the content property or class representatives of the target sample, and start the reconstruction from an image with certain semantic information. The success of such attacks depends on the probability of a successful guess, which becomes lower with higher complexity or variety of images. Second, the adversary may design attack techniques instead of optimizing the distance between the real and dummy gradients. We leave these advanced attacks as future work. 

\bheading{Defending other domains.} In this paper, we focus on the computer vision domain and image classification tasks. The reconstruction attacks may occur in other domains, e.g., natural language processing \cite{zhu2019deep}. Then the searched image transformations cannot be applied. However, it is possible to use text augmentation techniques \cite{kobayashi2018contextual,wei2019eda} (e.g., deletion, insertion, shuffling, synonym replacement) to preprocess the sensitive text to be less leaky without losing the semantics. Future work will focus on the design of an automatic search method for privacy protection of NLP tasks.

\section{Conclusion}
In this paper, we devise a novel methodology to automatically and efficiently search for data augmentation policies, which can prevent information leakage from the shared gradients. Our extensive evaluations demonstrate that the identified policies can defeat existing reconstruction attacks with negligible overhead. These policies also enjoy high transferability across different datasets, and applicability to different learning systems. We expect our search method can be adopted by researchers and practitioners to identify more effective policies when new data augmentation techniques are designed in the future.

\section{Acknowledgement}
We thank the anonymous reviewers for their valuable comments. This research was conducted in collaboration with SenseTime. This work is supported by A*STAR through the Industry Alignment Fund --- Industry Collaboration Projects Grant. It is also supported by the National Research Foundation, Singapore under its AI Singapore Programme (AISG Award No: AISG2-RP-2020-019).


\newpage

{\small
\bibliographystyle{ieee_fullname}
\bibliography{egbib}

\begin{thebibliography}{10}\itemsep=-1pt

\bibitem{abadi2016deep}
Martin Abadi, Andy Chu, Ian Goodfellow, H~Brendan McMahan, Ilya Mironov, Kunal
  Talwar, and Li Zhang.
\newblock Deep learning with differential privacy.
\newblock In {\em ACM SIGSAC conference on Computer and Communications
  Security}, 2016.

\bibitem{dutta2019discrepancy}
Dutta Aritra, Houcine~Bergou El, M.~Abdelmoniem Ahmed, Ho Chen-Yu, Narayan~Sahu
  Atal, Canini Marco, and Kalnis Panos.
\newblock On the discrepancy between the theoretical analysis and practical
  implementations of compressed communication for distributed deep learning.
\newblock In {\em AAAI}, 2019.

\bibitem{brisimi2018federated}
Theodora~S Brisimi, Ruidi Chen, Theofanie Mela, Alex Olshevsky, Ioannis~Ch
  Paschalidis, and Wei Shi.
\newblock Federated learning of predictive models from federated electronic
  health records.
\newblock {\em International Journal of Medical Informatics}, 2018.

\bibitem{iccv2019autoaugment}
Ekin~D Cubuk, Barret Zoph, Dandelion Mane, Vijay Vasudevan, and Quoc~V Le.
\newblock Autoaugment: Learning augmentation strategies from data.
\newblock In {\em Int. Conf. Comput. Vis.}, 2019.

\bibitem{deng2009imagenet}
Jia Deng, Wei Dong, Richard Socher, Li-Jia Li, Kai Li, and Li Fei-Fei.
\newblock Imagenet: A large-scale hierarchical image database.
\newblock In {\em IEEE Conf. Comput. Vis. Pattern Recog.}, 2009.

\bibitem{fan2020rethinking}
Lixin Fan, Kam~Woh Ng, Ce Ju, Tianyu Zhang, Chang Liu, Chee~Seng Chan, and
  Qiang Yang.
\newblock Rethinking privacy preserving deep learning: How to evaluate and
  thwart privacy attacks.
\newblock {\em arXiv preprint arXiv:2006.11601}, 2020.

\bibitem{geiping2020inverting}
Jonas Geiping, Hartmut Bauermeister, Hannah Dr{\"o}ge, and Michael Moeller.
\newblock Inverting gradients--how easy is it to break privacy in federated
  learning?
\newblock In {\em Advances in Neural Information Processing Systems}, 2020.

\bibitem{guo2020hidden}
Shangwei Guo, Tianwei Zhang, Han Qiu, Yi Zeng, Tao Xiang, and Yang Liu.
\newblock The hidden vulnerability of watermarking for deep neural networks.
\newblock {\em arXiv preprint arXiv:2009.08697}, 2020.

\bibitem{guo2020differentially}
Shangwei Guo, Tianwei Zhang, Tao Xiang, and Yang Liu.
\newblock Differentially private decentralized learning.
\newblock {\em arXiv preprint arXiv:2006.07817}, 2020.

\bibitem{guo2020towards}
Shangwei Guo, Tianwei Zhang, Xiaofei Xie, Lei Ma, Tao Xiang, and Yang Liu.
\newblock Towards byzantine-resilient learning in decentralized systems.
\newblock {\em arXiv preprint arXiv:2002.08569}, 2020.

\bibitem{hao2019efficient}
Meng Hao, Hongwei Li, Xizhao Luo, Guowen Xu, Haomiao Yang, and Sen Liu.
\newblock Efficient and privacy-enhanced federated learning for industrial
  artificial intelligence.
\newblock {\em IEEE Transactions on Industrial Informatics}, 16(10), 2019.

\bibitem{he2016deep}
Kaiming He, Xiangyu Zhang, Shaoqing Ren, and Jian Sun.
\newblock Deep residual learning for image recognition.
\newblock In {\em IEEE Conf. Comput. Vis. Pattern Recog.}, 2016.

\bibitem{he2019model}
Zecheng He, Tianwei Zhang, and Ruby~B Lee.
\newblock Model inversion attacks against collaborative inference.
\newblock In {\em ACM Annual Computer Security Applications Conference}, 2019.

\bibitem{he2020attacking}
Zecheng He, Tianwei Zhang, and Ruby~B Lee.
\newblock Attacking and protecting data privacy in edge-cloud collaborative
  inference systems.
\newblock {\em IEEE Internet of Things Journal}, 2020.

\bibitem{hitaj2017deep}
Briland Hitaj, Giuseppe Ateniese, and Fernando Perez-Cruz.
\newblock Deep models under the {GAN}: information leakage from collaborative
  deep learning.
\newblock In {\em ACM SIGSAC Conference on Computer and Communications
  Security}, 2017.

\bibitem{hore2010image}
Alain Hore and Djemel Ziou.
\newblock Image quality metrics: {PSNR} vs. {SSIM}.
\newblock In {\em International Conference on Pattern Recognition}, 2010.

\bibitem{mellor2020neural}
Mellor Joseph, Turner Jack, Storkey Amos, and J.~Crowley Elliot.
\newblock Neural architecture search without training.
\newblock {\em arXiv preprint arXiv:2006.04647}, 2020.

\bibitem{kobayashi2018contextual}
Sosuke Kobayashi.
\newblock Contextual augmentation: Data augmentation by words with paradigmatic
  relations.
\newblock {\em arXiv preprint arXiv:1805.06201}, 2018.

\bibitem{krizhevsky2009learning}
Alex Krizhevsky, Geoffrey Hinton, et~al.
\newblock Learning multiple layers of features from tiny images.
\newblock 2009.

\bibitem{kwok2005evolutionary}
Ngai~Ming Kwok, Gu Fang, and Weizhen Zhou.
\newblock Evolutionary particle filter: Re-sampling from the genetic algorithm
  perspective.
\newblock In {\em International Conference on Intelligent Robots and Systems},
  2005.

\bibitem{lecuyer2019certified}
Mathias Lecuyer, Vaggelis Atlidakis, Roxana Geambasu, Daniel Hsu, and Suman
  Jana.
\newblock Certified robustness to adversarial examples with differential
  privacy.
\newblock In {\em EEE Symposium on Security and Privacy}, 2019.

\bibitem{liu1989limited}
Dong~C Liu and Jorge Nocedal.
\newblock On the limited memory {BFGS} method for large scale optimization.
\newblock {\em Mathematical programming}, 45(1-3), 1989.

\bibitem{melis2019exploiting}
Luca Melis, Congzheng Song, Emiliano De~Cristofaro, and Vitaly Shmatikov.
\newblock Exploiting unintended feature leakage in collaborative learning.
\newblock In {\em IEEE Symposium on Security and Privacy}, 2019.

\bibitem{nasr2019comprehensive}
Milad Nasr, Reza Shokri, and Amir Houmansadr.
\newblock Comprehensive privacy analysis of deep learning: Passive and active
  white-box inference attacks against centralized and federated learning.
\newblock In {\em IEEE Symposium on Security and Privacy}, 2019.

\bibitem{neal2001annealed}
Radford~M Neal.
\newblock Annealed importance sampling.
\newblock {\em Statistics and Computing}, 2001.

\bibitem{niknam2020federated}
Solmaz Niknam, Harpreet~S Dhillon, and Jeffrey~H Reed.
\newblock Federated learning for wireless communications: Motivation,
  opportunities, and challenges.
\newblock {\em IEEE Communications Magazine}, 58(6), 2020.

\bibitem{paszke2019pytorch}
Adam Paszke, Sam Gross, Francisco Massa, Adam Lerer, James Bradbury, Gregory
  Chanan, Trevor Killeen, Zeming Lin, Natalia Gimelshein, Luca Antiga, et~al.
\newblock Pytorch: An imperative style, high-performance deep learning library.
\newblock In {\em Advances in Neural Information Processing Systems}, 2019.

\bibitem{phan2020scalable}
Hai Phan, My~T Thai, Han Hu, Ruoming Jin, Tong Sun, and Dejing Dou.
\newblock Scalable differential privacy with certified robustness in
  adversarial learning.
\newblock In {\em International Conference on Machine Learning}, 2020.

\bibitem{pool}
Philip Popien.
\newblock {AutoAugment - Learning Augmentation Policies from Data}.
\newblock \url{https://github.com/DeepVoltaire/AutoAugment}.

\bibitem{qiu2020fencebox}
Han Qiu, Yi Zeng, Tianwei Zhang, Yong Jiang, and Meikang Qiu.
\newblock Fencebox: A platform for defeating adversarial examples with data
  augmentation techniques.
\newblock {\em arXiv preprint arXiv:2012.01701}, 2020.

\bibitem{qiu2020mitigating}
Han Qiu, Yi Zeng, Qinkai Zheng, Tianwei Zhang, Meikang Qiu, and Gerard Memmi.
\newblock Mitigating advanced adversarial attacks with more advanced gradient
  obfuscation techniques.
\newblock {\em arXiv preprint arXiv:2005.13712}, 2020.

\bibitem{qiu2020towards}
Han Qiu, Qinkai Zheng, Tianwei Zhang, Meikang Qiu, Gerard Memmi, and Jialiang
  Lu.
\newblock Towards secure and efficient deep learning inference in dependable
  iot systems.
\newblock {\em IEEE Internet of Things Journal}, 2020.

\bibitem{wei2019eda}
Jason Wei and Kai Zou.
\newblock Eda: Easy data augmentation techniques for boosting performance on
  text classification tasks.
\newblock {\em arXiv preprint arXiv:1901.11196}, 2019.

\bibitem{wei2020framework}
Wenqi Wei, Ling Liu, Margaret Loper, Ka-Ho Chow, Mehmet~Emre Gursoy, Stacey
  Truex, and Yanzhao Wu.
\newblock A framework for evaluating gradient leakage attacks in federated
  learning.
\newblock {\em arXiv preprint arXiv:2004.10397}, 2020.

\bibitem{wistuba2019survey}
Martin Wistuba, Ambrish Rawat, and Tejaswini Pedapati.
\newblock A survey on neural architecture search.
\newblock {\em arXiv preprint arXiv:1905.01392}, 2019.

\bibitem{xiao2017fashion}
Han Xiao, Kashif Rasul, and Roland Vollgraf.
\newblock Fashion-mnist: a novel image dataset for benchmarking machine
  learning algorithms.
\newblock {\em arXiv preprint arXiv:1708.07747}, 2017.

\bibitem{yang2019federated}
Qiang Yang, Yang Liu, Tianjian Chen, and Yongxin Tong.
\newblock Federated machine learning: Concept and applications.
\newblock {\em ACM Transactions on Intelligent Systems and Technology}, 10(2),
  2019.

\bibitem{zeng2020deepsweep}
Yi Zeng, Han Qiu, Shangwei Guo, Tianwei Zhang, Meikang Qiu, and Bhavani
  Thuraisingham.
\newblock Deepsweep: An evaluation framework for mitigating dnn backdoor
  attacks using data augmentation.
\newblock In {\em arXiv preprint arXiv:2012.07006}, 2021.

\bibitem{zhao2020idlg}
Bo Zhao, Konda~Reddy Mopuri, and Hakan Bilen.
\newblock {iDLG}: Improved deep leakage from gradients.
\newblock {\em arXiv preprint arXiv:2001.02610}, 2020.

\bibitem{zhao2020privatedl}
Qi Zhao, Chuan Zhao, Shujie Cui, Shan Jing, and Zhenxiang Chen.
\newblock {PrivateDL}: Privacy-preserving collaborative deep learning against
  leakage from gradient sharing.
\newblock {\em International Journal of Intelligent Systems}, 2020.

\bibitem{zhu2019deep}
Ligeng Zhu, Zhijian Liu, and Song Han.
\newblock Deep leakage from gradients.
\newblock In {\em Advances in Neural Information Processing Systems}, 2019.

\bibitem{zoph2016neural}
Barret Zoph and Quoc~V Le.
\newblock Neural architecture search with reinforcement learning.
\newblock {\em arXiv preprint arXiv:1611.01578}, 2016.

\end{thebibliography}
}

\clearpage
\appendix
\section{Implementation Details}
\label{sec:implementation-detail}

\bheading{Searching Transformation Policies.}
Instead of each participant searching for their own transformation policies, we first adopt Algorithm 1 to obtain a universal policy set $\mathcal{T}$ on a validation set. During the training process, all participants would use $\mathcal{T}$ for privacy protection. In particular, for each policy $T$, we calculate the corresponding privacy score $S_{pri}(T)$ on 100 randomly selected images in the validation set except the first 100 images that are used for attack evaluation. 
We optimize the randomly initialized model $f$ for 10 forward-background rounds and use the average value of the accuracy scores of the ten rounds as $S_{acc}$. The batch size of each round is set as 128. 
We adaptively adjust the accuracy threshold $T_{acc}$ for different architectures and datasets. In particular, we set $T_{acc}$ as $-85$ (ResNet20, CIFAR100), $-80$ (ConvNet, CIFAR100), $-12$ (ResNet20, F-MNIST), and $-10$ (ConvNet, F-MNIST), respectively.

\bheading{Training Implementation}
We utilize SGD with momentum 0.9 and weight decay $5\cdot 10^{-4}$ to optimize the deep neural networks. We set the training epoch as 200 (resp. 100) for CIFAR100 (resp. F-MNIST). 
The initial learning rate is 0.1 and steply decays by a factor of 0.1 at $\frac{3}{8}$, $\frac{5}{8}$, and $\frac{7}{8}$ of all iterations.

\bheading{Attacking Implementation}
For attacks without using L-BFGS, we follow the same settings in~\cite{geiping2020inverting}. We set the iteration number of the reconstruction optimizations as 4800 and adopt the same training policy as the above collaborative learning.
The total variation weight is set as $10^{-4}$. 

Because The results of L-BFGS-based reconstruction attacks are unstable, we run the attacks 16 times and select the best reconstruction results. For fair comparison, we reduce the iteration number to 300. 

\bheading{Figure Implementation}
We present the implementation details of the figures in our manuscript as follows:
\begin{packeditemize}
\item We select $3-1-7$ as the privacy-aware transformation policy to generate Figure~4-2. We adopt a semi-train ResNet20 on CIFAR100 to calculate the $GradSim$ values of the interpolation images.
\item In Figure~4-3, we randomly sample 100 different transformation policies from the 127, 550 policies an calculate the average $PSNR$ of reconstructed images under IG~\cite{geiping2020inverting}. We adopt the semi-train ResNet20 as the initial model to calculate the gradients on the first 100 images in the validation set.
\item The reconstruction attack in Figure 5-7 is IG~\cite{geiping2020inverting}. The random transformation policy is $19-1-18$ and the searched policy is the hybrid of $3-1-7$ and $43-18-18$. 
\end{packeditemize}




\section{Transformation Space}
\label{sec:augmentation-detail}
We summarize the 50 transformations used in our manuscript in Table~\ref{tab:transformation-detail}.

\begin{table*}
    \centering
    \captionsetup[subtable]{position = below}
           \resizebox{0.16\textwidth}{!}{
           \begin{subtable}{0.25\linewidth}
               \centering
               \begin{tabular}{ccc}
                   \toprule
                   \textbf{Index} & \textbf{Transformation} & \textbf{Magnitude} \\ \midrule
                        0 & invert & 7 \\
                        1 & contrast & 6 \\
                        2 & rotate & 2 \\
                        3 & translateX & 9 \\
                        4 & sharpness & 1 \\
                        5 & sharpness & 3 \\
                        6 & shearY & 2 \\
                        7 & translateY & 2 \\
                        8 & autocontrast & 5 \\
                        9 & equalize & 2 \\
                        10 & shearY & 5 \\
                        11 & posterize & 5 \\
                        12 & color & 3 \\ \bottomrule
               \end{tabular}
           \end{subtable}}%
           \hspace*{3em}
           \resizebox{0.16\textwidth}{!}{
           \begin{subtable}{0.25\linewidth}
               \centering
               \begin{tabular}{|ccc}
                   \toprule
                    \textbf{Index} & \textbf{Transformation} & \textbf{Magnitude} \\ \midrule
                        13 & brightness & 5 \\
                        14 & sharpness & 9 \\
                        15 & brightness & 9 \\
                        16 & equalize & 5 \\
                        17 & equalize & 1 \\
                        18 & contrast & 7 \\
                        19 & sharpness & 5 \\
                        20 & color & 5 \\
                        21 & translateX & 5 \\
                        22 & equalize & 7 \\
                        23 & autocontrast & 8 \\
                        24 & translateY & 3 \\
                        25 & sharpness & 6 \\ \bottomrule
               \end{tabular}
           \end{subtable}}%
           \hspace*{3em}
           \resizebox{0.16\textwidth}{!}{
           \begin{subtable}{0.25\linewidth}
               \centering
               \begin{tabular}{|ccc}
                   \toprule
                    \textbf{Index} & \textbf{Transformation} & \textbf{Magnitude} \\ \midrule
                        26 & brightness & 6 \\
                        27 & color & 8 \\
                        28 & solarize & 0 \\
                        29 & invert & 0 \\
                        30 & equalize & 0 \\
                        31 & autocontrast & 0 \\
                        32 & equalize & 8 \\
                        33 & equalize & 4 \\
                        34 & color & 5 \\
                        35 & equalize & 5 \\
                        36 & autocontrast & 4 \\
                        37 & solarize & 4 \\
                        38 & brightness & 3 \\ \bottomrule
               \end{tabular}
           \end{subtable}}%
            \hspace*{3em}
           \resizebox{0.16\textwidth}{!}{
           \begin{subtable}{0.25\linewidth}
               \centering
               \begin{tabular}{|ccc}
                   \toprule
                    \textbf{Index} & \textbf{Transformation} & \textbf{Magnitude} \\ \midrule
                        39 & color & 0 \\
                        40 & solarize & 1 \\
                        41 & autocontrast & 0 \\
                        42 & translateY & 3 \\
                        43 & translateY & 4 \\
                        44 & autocontrast & 1 \\
                        45 & solarize & 1 \\
                        46 & equalize & 5 \\
                        47 & invert & 1 \\
                        48 & translateY & 3 \\
                        49 & autocontrast & 1 \\ 
                         & & \\
                         & & \\ \bottomrule
               \end{tabular}
           \end{subtable}}%
           
            \captionsetup[table]{position=bottom}
           \caption{Summary of the 50 transformations.}
           \label{tab:transformation-detail}

\end{table*}

\section{More Transformation Results}
We provide more experimental results of other high-ranking transformations in Tab~\ref{tab:other-results}. The reconstruction attack is IG~\cite{geiping2020inverting}.

\begin{table*}
    \centering
    \captionsetup[subtable]{position = below}
           \resizebox{0.21\textwidth}{!}{
           \begin{subtable}{0.22\linewidth}
               \centering
               \begin{tabular}{c|cc}
                   \hline
                   \textbf{Policy} & \textbf{PSNR} & \textbf{ACC} \\ \hline
                        3-18-28 &	7.3 &	72.09 \\
                        7-3  & 7.64 & 71.63 \\ 
                        \hline
               \end{tabular}
               \caption{CIFAR100 with ResNet20}
           \end{subtable}}%
           \hspace*{2em}
           \resizebox{0.21\textwidth}{!}{
           \begin{subtable}{0.22\linewidth}
               \centering
               \begin{tabular}{c|cc}
                   \hline
                   \textbf{Policy} & \textbf{PSNR} & \textbf{ACC} \\ \hline
                        15-43 & 8.27  & 68.66 \\
                        37-33-3 & 7.83  & 67.89 \\ \hline
               \end{tabular}
               \caption{CIFAR100 with ConvNet}
           \end{subtable}}%
           \hspace*{2em}
           \resizebox{0.21\textwidth}{!}{
           \begin{subtable}{0.22\linewidth}
               \centering
               \begin{tabular}{c|cc}
                   \hline
                   \textbf{Policy} & \textbf{PSNR} & \textbf{ACC} \\ \hline
                        15-40-5 & 7.82 &	88.03 \\
                        0-39-35 &   6.97  & 87.10 \\ \hline
               \end{tabular}
               \caption{F-MNIST with ResNet20}
           \end{subtable}}%
           \hspace*{2em}
           \resizebox{0.21\textwidth}{!}{
           \begin{subtable}{0.22\linewidth}
               \centering
               \begin{tabular}{c|cc}
                   \hline
                   \textbf{Policy} & \textbf{PSNR} & \textbf{ACC} \\ \hline
                        43-43-48 &	7.00  & 88.08 \\
                        42-26-45 &   7.51  & 87.91 \\ \hline
               \end{tabular}
               \caption{F-MNIST with ConvNet}
           \end{subtable}}%
            \captionsetup[table]{position=bottom}
           \caption{PSNR (db) and model accuracy (\%) of different transformation configurations for each architecture and dataset.}
           \label{tab:other-results}

\end{table*}

\begin{figure*}
     \centering
     \begin{subfigure}[b]{0.3\linewidth}
         \centering

        \includegraphics[ width=\linewidth]{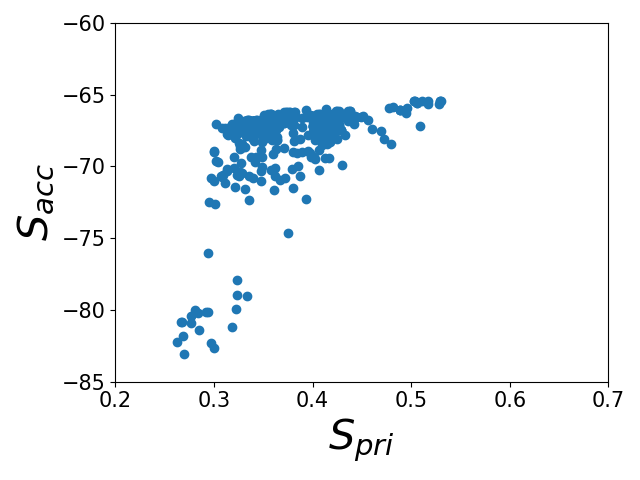}  
        
         \caption{$S_{pri} \in [0.262, 0.530]$}
         \label{fig:privacy-accuracy-distribution1}
     \end{subfigure}
     \begin{subfigure}[b]{0.3\linewidth}
         \centering
        \includegraphics[ width=\linewidth]{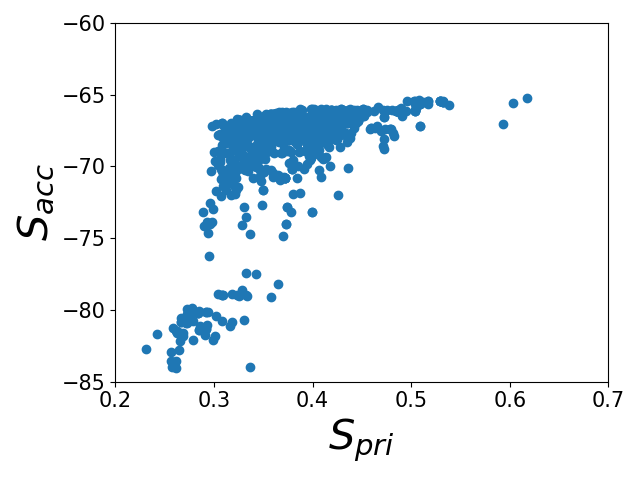} 

         \caption{$S_{pri} \in [0.231, 0.617]$}
         \label{fig:privacy-accuracy-distribution2}
     \end{subfigure}
     \begin{subfigure}[b]{0.3\linewidth}
         \centering
        \includegraphics[ width=\linewidth]{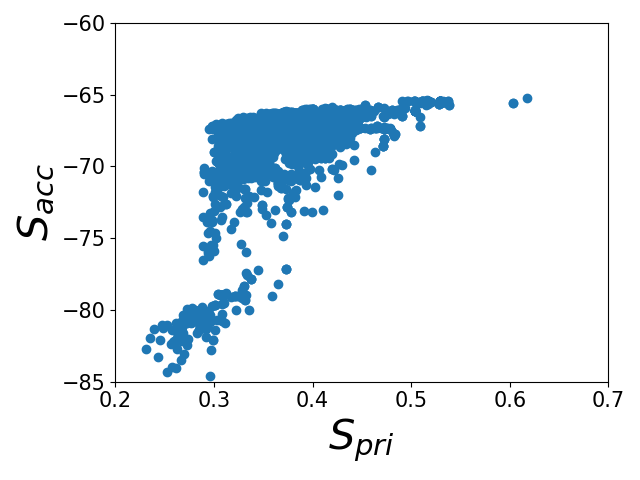}
         \caption{$S_{pri} \in [0.231, 0.617]$}
         \label{fig:privacy-accuracy-distribution3}
     \end{subfigure}

        \caption{$S_{pri}-S_{acc}$ distributes of different numbers of policies.}
        \label{fig:privacy-accuracy-distribution}
\end{figure*}
\section{Recap Search Process}
\label{sec:search-process}

We adopt a simple yet effective search strategy, i.e., \emph{random search}, to find satisfactory transformation policies.
More intelligent search methods, such as evolution algorithm and reinforcement learning, could mitigate meaningless candidates and improve the search efficiency. 
We also investigate the affect of the number of policy candidates to the proposed defense. We test 500, 1500, and 5000 policies and plot the privacy-accuracy distributions in Figure~\ref{fig:privacy-accuracy-distribution}. We observe the distributions of 1500 and 5000 candidates are similar and the privacy score ranges are the same.
This indicates that 1500 policies are enough for searching satisfactory transformation policies.

\section{Proof of Theorem 1}
\begin{theorem}
    Consider two transformation policies $T_1$, $T_2$. Let $x$ be a training sample, and $x^*_1, x^*_2$ be the reconstructed samples via Equation 2 with $T_1$ and $T_2$ respectively. If ${\small \emph{\texttt{GradSim}}(x'(i), T_1(x)) \geq \emph{\texttt{GradSim}}(x'(i), T_2(x))}$ is satisfied for $\forall i \in [0, 1]$, then $T_2$ is more effective against reconstruction attacks than $T_1$, i.e.,
    \begin{equation}\nonumber
    \emph{Pr}[\emph{\texttt{PSNR}}(x^*_1, T_1(x))\geq \epsilon] \geq \emph{Pr}[\emph{\texttt{PSNR}}(x^*_2, T_2(x))\geq \epsilon].
    \end{equation}
\end{theorem}
\begin{proof}
    Since distance-based reconstruction attacks adopt gradient decent based optimization algorithms, the distance between two adjacent points $f_{T}(i)$, $f_{T}(i+\bigtriangleup i)$ is positively correlated with $Pr(T, x'(i), x'(i+\bigtriangleup i))$, where $Pr(T, x'(i), x'(i+\bigtriangleup i))$ is the probability that the adversary can reconstruct $x'(i+\bigtriangleup i)$ from $x'(i)$, $$x'(i+\bigtriangleup i) = \frac{K-i-\bigtriangleup i}{K} * x_0 + \frac{i+\bigtriangleup i}{K} * T(x).$$ In particular, if $|f_{T}(i) - f_{T}(i+\bigtriangleup i)|$ is larger, the gradient decent based optimization is more likely to update $x'(i)$ to $x'(i+\bigtriangleup i)$. 
    
    Without loss of generality, we assume the derivative of $f_{T}(i)$, $f_{T}'(i)$ is differentiable and $f_{T}'(i) \varpropto Pr(T, x'(i), x'(i+\bigtriangleup i))$.  Then, we have
    \begin{equation}
        f_{T}(i) = \int_{0}^{i} f_{T}'(z) \,\mathrm{d}z.
    \end{equation}
    and 
    \begin{equation}
        f_{T}(i) \varpropto \int_{0}^{i} Pr(T, x'(z), x'(z+\bigtriangleup z))\,\mathrm{d}z.
    \end{equation}

    Thus,
    \begin{equation}
        \int_{0}^{1} f_{T}(i) \,\mathrm{d}i \varpropto \int_{0}^{1} \int_{0}^{i} Pr(T, x'(z), x'(z+\bigtriangleup z))\,\mathrm{d}z\mathrm{d}i,
    \end{equation}
    where $$\emph{Pr}[\emph{\texttt{PSNR}}(x^*, T(x))\geq \epsilon] = \int_{0}^{1} \int_{0}^{i} Pr(T, x'(z), x'(z+\bigtriangleup z))\,\mathrm{d}z\mathrm{d}i.$$

    Since $f_{T_1}(i) \geq f_{T_2}(i)$ for $\forall i \in [0, 1]$, 
    \begin{equation}
        \int_{0}^{1} f_{T_1}(i) \,\mathrm{d}i \geq \int_{0}^{1} f_{T_2}(i) \,\mathrm{d}i,
    \end{equation}
    and therefore, 
    \begin{equation}
        \emph{Pr}[\emph{\texttt{PSNR}}(x^*_1, T_1(x))\geq \epsilon] \geq \emph{Pr}[\emph{\texttt{PSNR}}(x^*_2, T_2(x))\geq \epsilon].
    \end{equation}
\end{proof}

\section{Adaptive Attack}

Actually the general adaptive attack towards our privacy preserving is not easy to design, we appreciate one of reviewers points out that 
 the adversary can possibly initialize the starting image $x_{0}$ as black pixels when the shift transformation is applied. The results are shown in \ref{tab:cifar100-zero}
 
\begin{table}[htp]
\begin{center}
\resizebox{0.45\textwidth}{!}{
\begin{tabular}{|c|c|c|c|}
\hline
\textbf{Model}     & \textbf{Transformation}  & \textbf{Initialized to} & \textbf{Initialized to} \\
& \textbf{Policy} & \textbf{random values} & \textbf{black pixels} \\\hline
ResNet20  & 3-1-7   & 6.58   & 6.64 \\ \hline
ConvNet & 21-13-3 & 5.76   & 5.78 \\ \hline
\end{tabular}}
\end{center}
\caption{PSNR (db) of different initialization for each architecture on CIFAR100.}
\label{tab:cifar100-zero}
\end{table}

\section{More Visualization Results}
We show more original images in CIFAR100 (Figure \ref{fig:more-images-ori}) and the corresponding reconstructed images under IG without any transformation (Figure \ref{fig:attack-images-no-transform}). We also illustrate the transformed original images with the hybrid policy adopted for protecting ResNet20 on CIFAR100 (Figure \ref{fig:more-images-transform}) and the corresponding reconstructed images under IG with the transformation policy (Figure \ref{fig:attack-images-transform}). The batch size of the above collaborative training processes is 1. We further visually show the corresponding figures (Figure \ref{fig:more-images-ori-8}-\ref{fig:attack-images-transform-8}) when the batch size is 8. In consideration of the expensive cost of attacking ImageNet~\cite{deng2009imagenet}, we only present several visualized images in \ref{fig:imagenet}.

\begin{figure*} [t]
    \begin{center}
    \begin{tabular}{c@{\hspace{0.01\linewidth}}c@{\hspace{0.01\linewidth}}c@{\hspace{0.01\linewidth}} c@{\hspace{0.01\linewidth}} }

        \includegraphics[ width=0.23\linewidth]{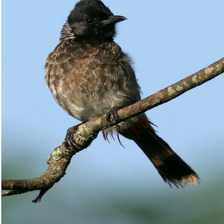} &
        \includegraphics[ width=0.23\linewidth]{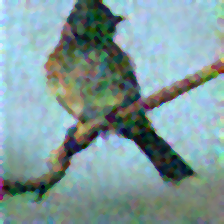} & 
        \includegraphics[ width=0.23\linewidth]{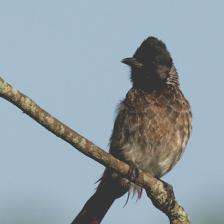} &
        \includegraphics[ width=0.23\linewidth]{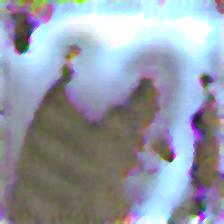}  \\

        \includegraphics[ width=0.23\linewidth]{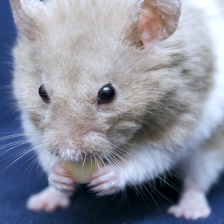} &
        \includegraphics[ width=0.23\linewidth]{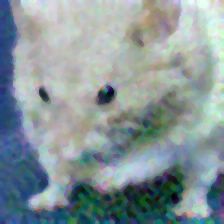} & 
        \includegraphics[ width=0.23\linewidth]{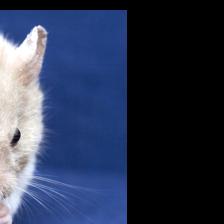} &
        \includegraphics[ width=0.23\linewidth]{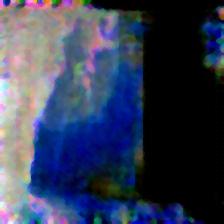} \\

        \includegraphics[ width=0.23\linewidth]{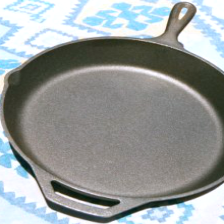} &
        \includegraphics[ width=0.23\linewidth]{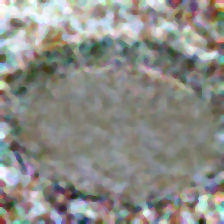} & 
        \includegraphics[ width=0.23\linewidth]{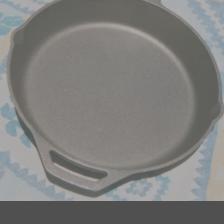} &
        \includegraphics[ width=0.23\linewidth]{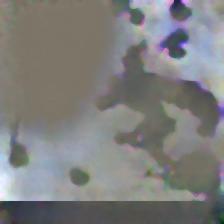} \\
        original & reconstruct & w/ transformation & recover from transformation
        
    \end{tabular}
    \end{center}
    
    \vspace{-1em}
    \caption{Visualization ImageNet results on untrained ResNet18 architecture. The transformation policy is 3-1-7+43-18-18.  }
    \label{fig:imagenet}
\end{figure*}

\clearpage
\begin{figure*}
    \centering
        \includegraphics[ width=0.8\linewidth]{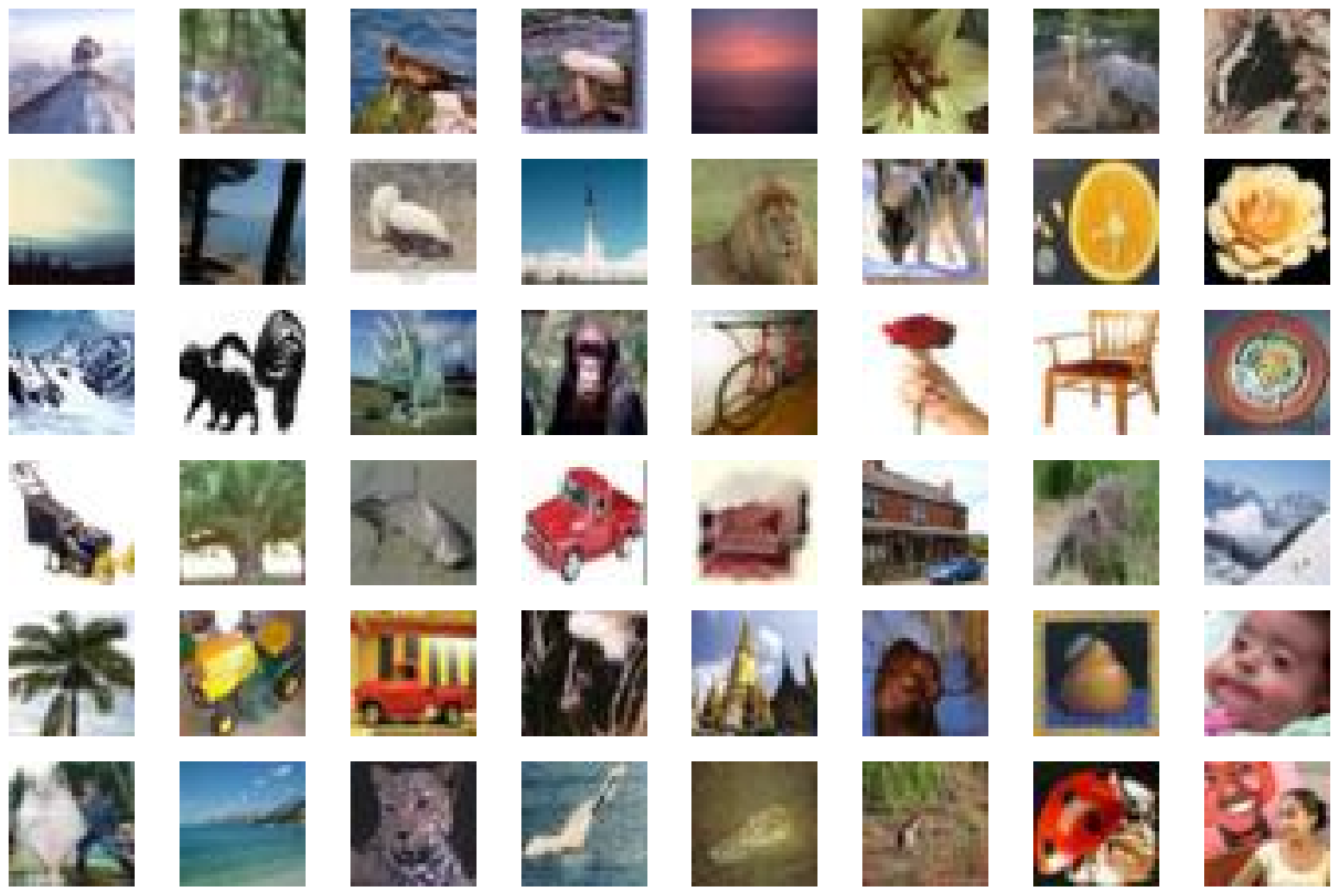}
    \caption{The original images in CIFAR100.} 
    \label{fig:more-images-ori}
\end{figure*}

\begin{figure*}
    \centering
        \includegraphics[ width=0.8\linewidth]{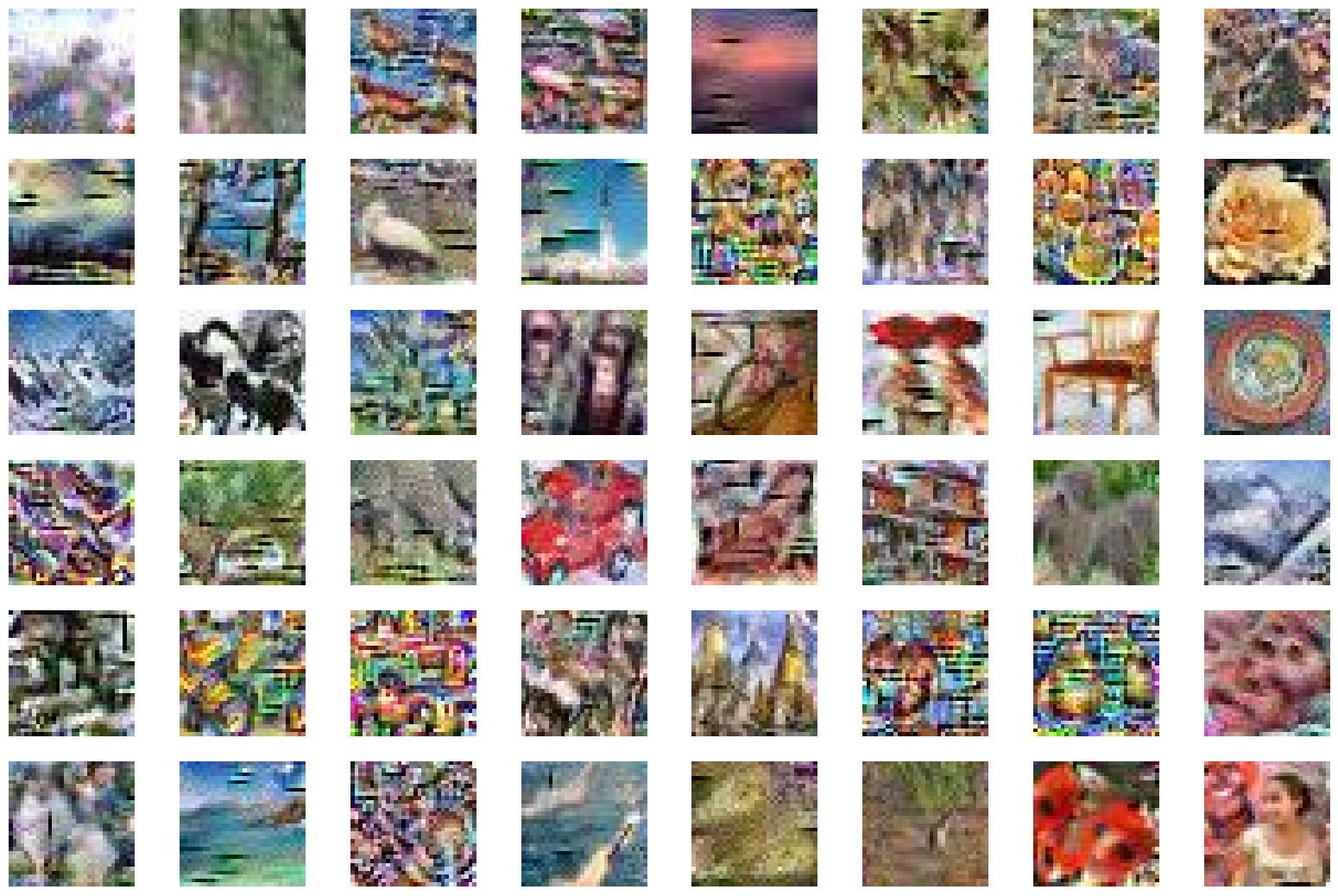}
    \caption{The reconstructed images of IG under ResNet20, CIFAR100 w/o transformation policy.} 
    \label{fig:attack-images-no-transform}
\end{figure*}

\begin{figure*}
    \centering
        \includegraphics[ width=0.8\linewidth]{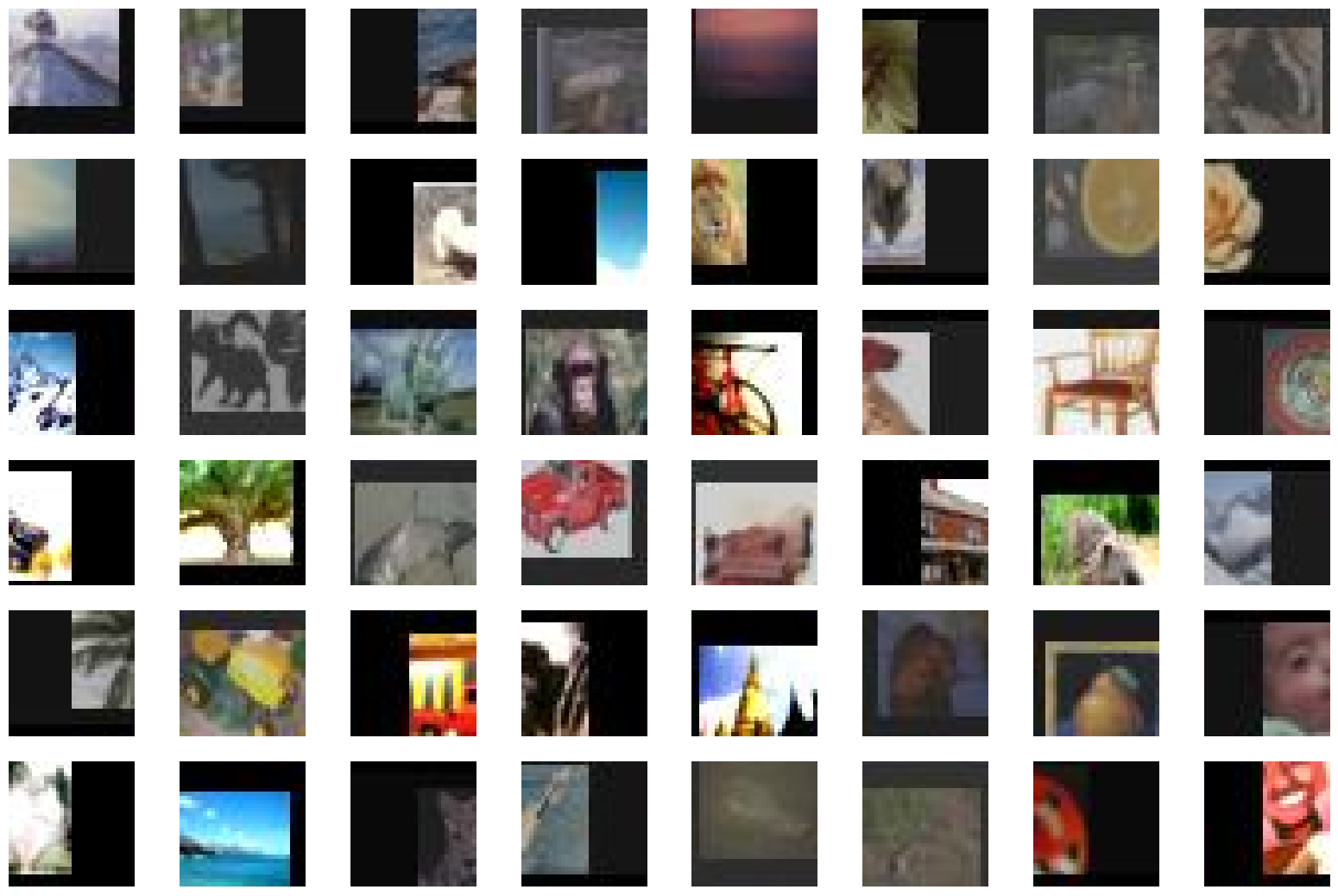}
    \caption{The original images in CIFAR100 w/ the privacy-preserving transformation policy.} 
    \label{fig:more-images-transform}
\end{figure*}

\begin{figure*}
    \centering
        \includegraphics[ width=0.8\linewidth]{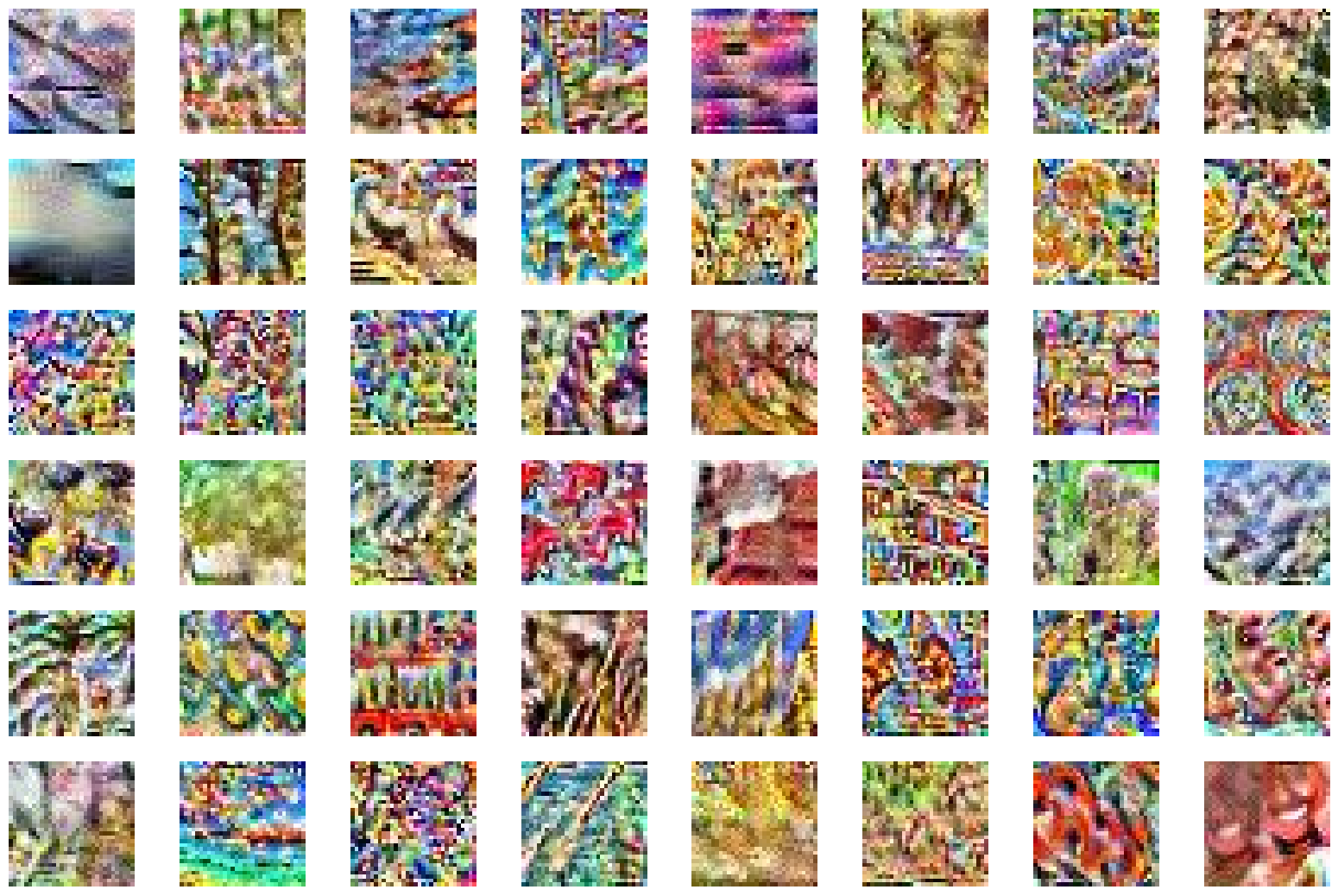}
    \caption{The reconstructed images of IG under ResNet20, CIFAR100 w/ the privacy-preserving transformation policy.} 
    \label{fig:attack-images-transform}
\end{figure*}

\begin{figure*}
    \centering
        \includegraphics[ width=0.8\linewidth]{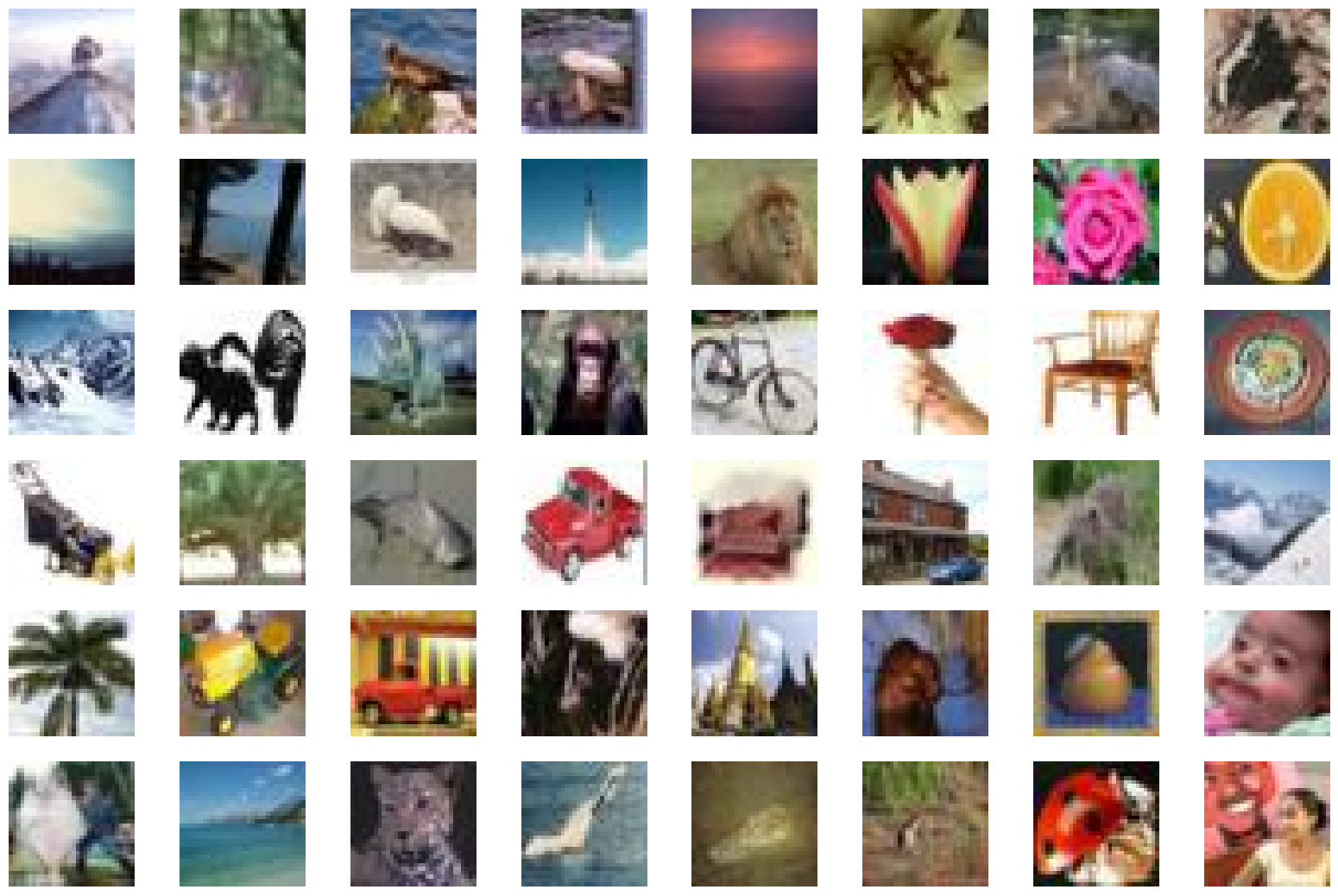}
    \caption{The original images in CIFAR100.} 
    \label{fig:more-images-ori-8}
\end{figure*}

\begin{figure*}
    \centering
        \includegraphics[ width=0.8\linewidth]{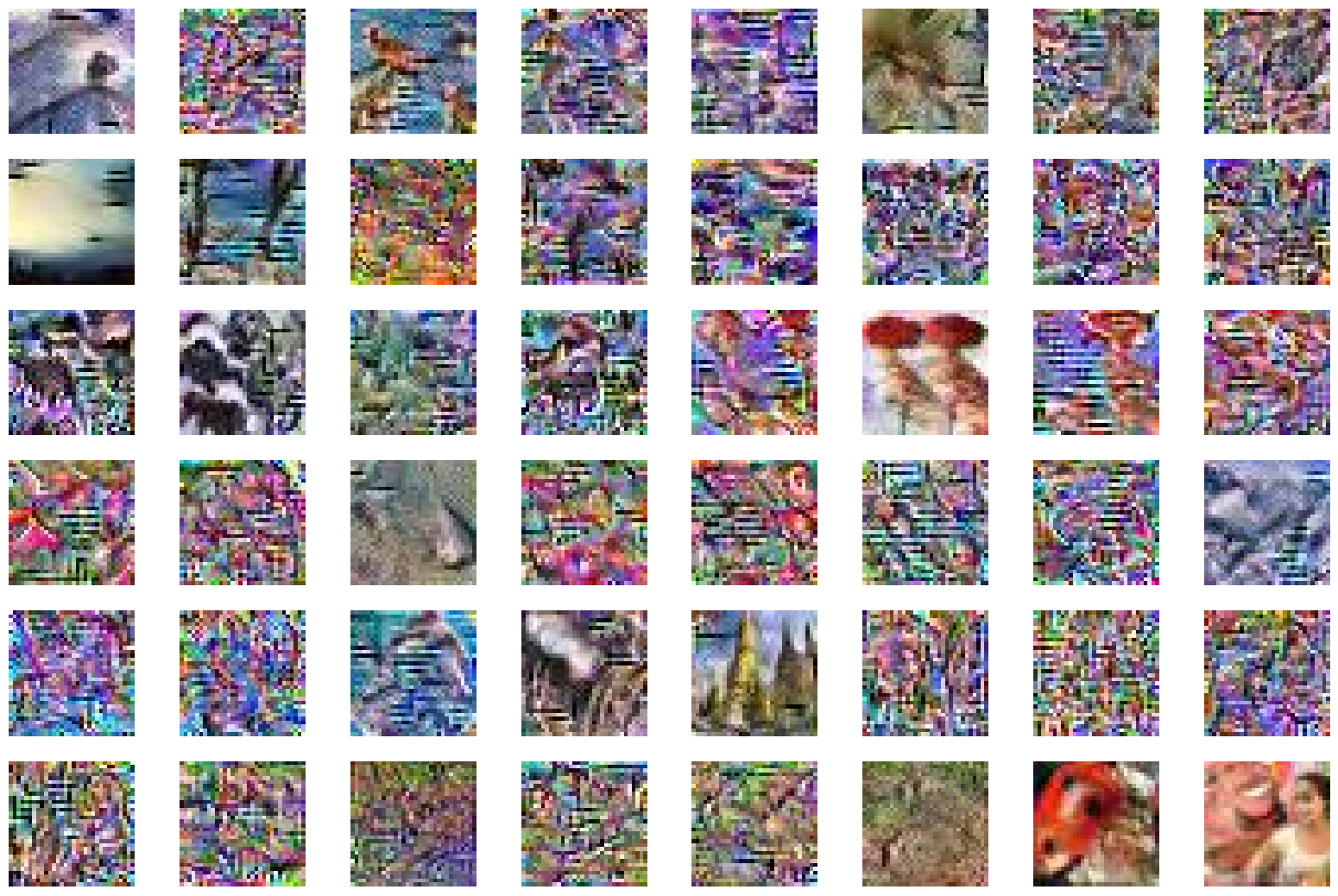}
    \caption{The reconstructed images of IG under ResNet20, CIFAR100 w/o transformation policy (batch size=8).} 
    \label{fig:attack-images-no-transform-8}
\end{figure*}

\begin{figure*}
    \centering
        \includegraphics[ width=0.8\linewidth]{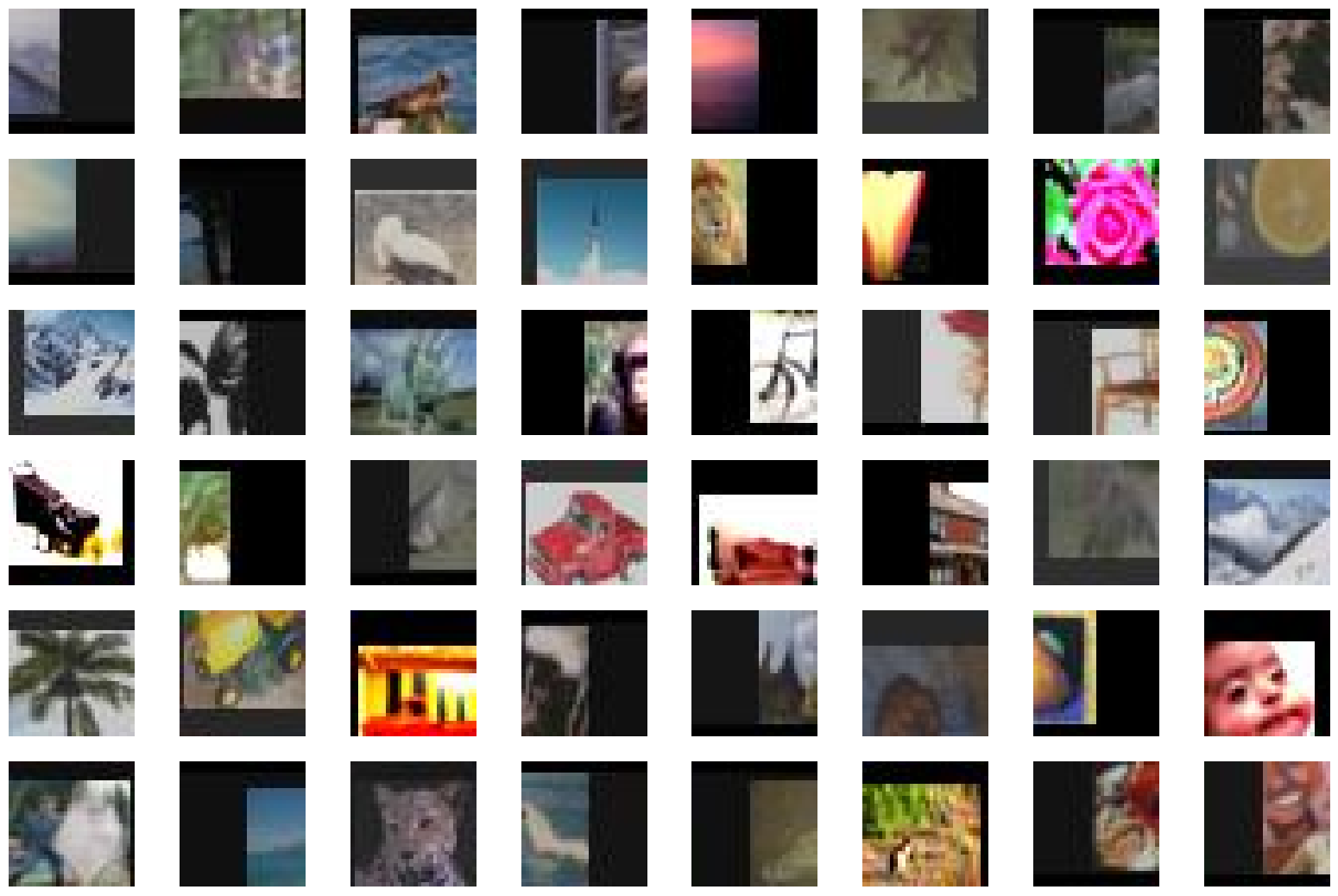}
    \caption{The original images in CIFAR100 w/ the privacy-preserving transformation policy.} 
    \label{fig:more-images-transform-8}
\end{figure*}

\begin{figure*}
    \centering
        \includegraphics[ width=0.8\linewidth]{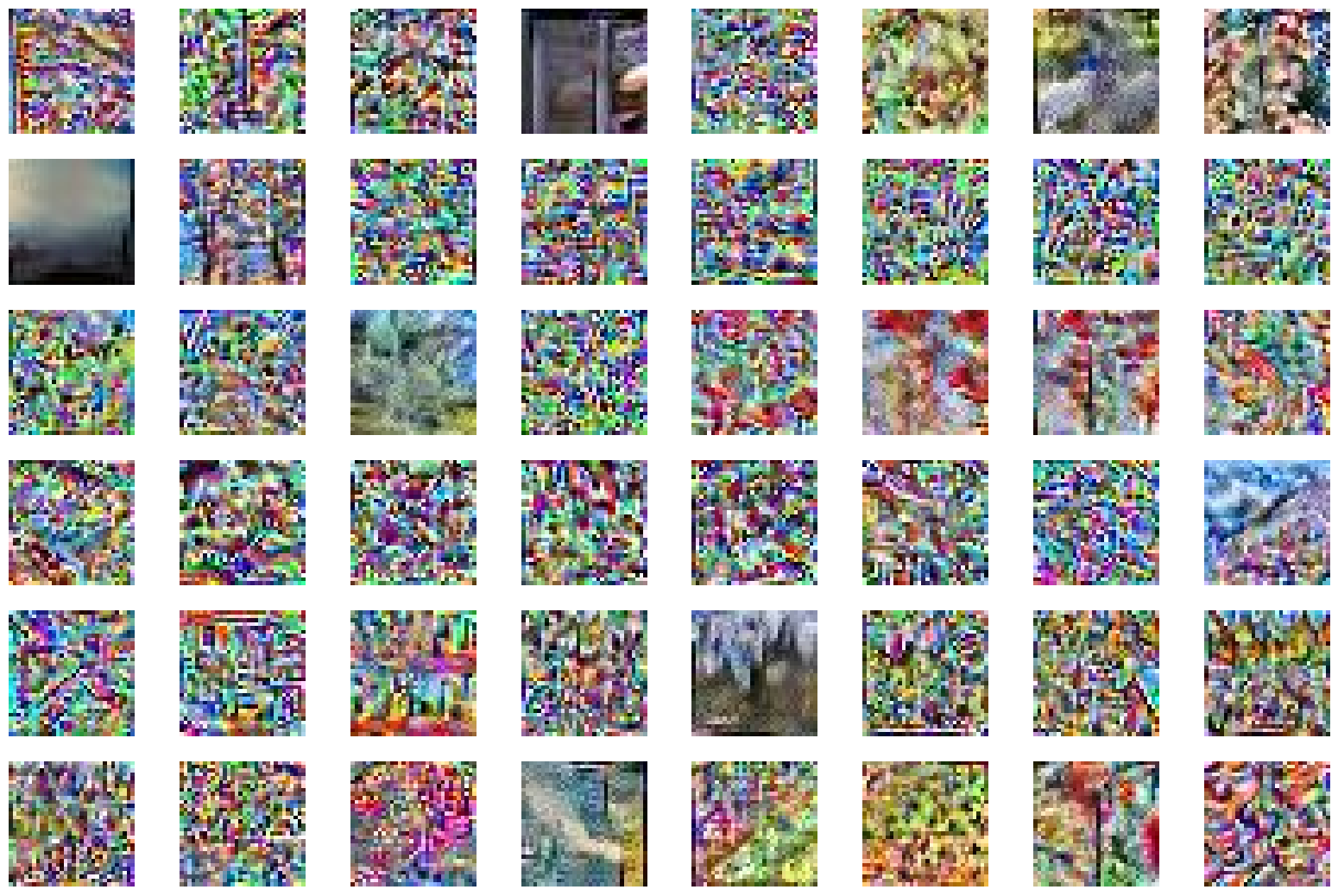}
    \caption{The reconstructed images of IG under ResNet20, CIFAR100 w/ the privacy-preserving transformation policy (batch size=8). } 
    \label{fig:attack-images-transform-8}
\end{figure*}

\end{document}